\def\full{1}
\ifnum\full=1
\documentclass[11pt]{article}
\else
\documentclass[10pt,twocolumn]{article}
\fi

\def\final{1}
\usepackage{times,color}
\ifnum\full=0 %shiva-are we allowed to do this?
\else
\usepackage{pifont,xspace,epsfig,wrapfig}
\usepackage{fullpage}
\usepackage[small,it]{caption}
\usepackage{graphicx}
\fi
\usepackage{amsmath, amsthm, amssymb}
\usepackage{latexsym}
\usepackage{amsfonts}
\ifnum\full=1
\usepackage{hyperref}
\fi

\newenvironment{remark}{\noindent {\em Remark:}}{\smallskip}

\DeclareSymbolFont{AMSb}{U}{msb}{m}{n}
\DeclareMathSymbol{\N}{\mathbin}{AMSb}{"4E}
\DeclareMathSymbol{\Z}{\mathbin}{AMSb}{"5A}
\DeclareMathSymbol{\R}{\mathbin}{AMSb}{"52}
\DeclareMathSymbol{\Q}{\mathbin}{AMSb}{"51}
\DeclareMathSymbol{\erert}{\mathbin}{AMSb}{"50}
\DeclareMathSymbol{\I}{\mathbin}{AMSb}{"49}
\DeclareMathSymbol{\C}{\mathbin}{AMSb}{"43}

\newcommand{\pac}{\ensuremath{\text{\sf PAC}}}
\newcommand{\ppac}{\ensuremath{\text{\sf PPAC}}} %sofya: priv-PAC sounds worse to me
\newcommand{\parity}{\ensuremath{\text{\sf PARITY}}}
\newcommand{\mparity}{\ensuremath{\text{\sf MASKED-PARITY}}}
\newcommand{\sq}{\ensuremath{\text{\sf SQ}}}
\newcommand{\nasq}{\ensuremath{\text{\sf NASQ}}}
\newcommand{\li}{\ensuremath{\text{\sf LI}}}
\newcommand{\lni}{\ensuremath{\text{\sf LNI}}}

\ifnum\final=0
\newcommand{\mynote}[1]{\marginpar{\tiny\sf #1}}
\else
\newcommand{\mynote}[1]{}
\fi
\newcommand{\anote}[1]{\mynote{ Adam: {#1}}}
\newcommand{\snote}[1]{\mynote{Shiva: {#1}}}
\newcommand{\rnote}[1]{\mynote{Sofya: {#1}}}

\newcommand{\ie}{{\it i.e.,\ }}
\newcommand{\eg}{{\it e.g.,\ }}
\newcommand{\etal}{{\it et al.\,}}
\newcommand{\through}{,\ldots,}
\newcommand{\AAA}{\mathcal A}
\newcommand{\BBB}{\mathcal B}
\newcommand{\CCC}{\mathcal C}
\newcommand{\HHH}{\mathcal H}
\newcommand{\DDD}{\mathcal D}
\newcommand{\XXX}{\mathcal X}

\newcommand{\remove}[1]{}
\newcommand{\eps}{\epsilon}
\newcommand{\z}{\mathrm{z}}
\newcommand{\zo}{\{0,1\}}
\newcommand{\oo}{\{+1,-1\}}

\def\E{\mathop{\mathbb{E}}\displaylimits}
\def\poly{\mathop{\rm{poly}}\nolimits}
\def\Lap{\mathop{\rm{Lap}}\nolimits}

\newtheorem{theorem}{Theorem}[section]
\newtheorem{lemma}[theorem]{Lemma}
\newtheorem{prop}[theorem]{Proposition}

\newtheorem{claim}[theorem]{Claim}

\newtheorem{corollary}[theorem]{Corollary}
\newtheorem{definition}[theorem]{Definition}

\newcommand{\alg}[2]{\begin{center}\fbox{\begin{minipage}{0.99\columnwidth}{\begin{center}\underline{\textsc{#1}}\end{center}{#2}}\end{minipage}}\end{center}}
\newcommand{\thmref}[1]{Theorem~\ref{thm:#1}}
\newcommand{\lemref}[1]{Lemma~\ref{lem:#1}}

\newcommand{\figref}[1]{Figure~\ref{fig:#1}}

\newcommand{\secref}[1]{Section~\ref{sec:#1}}
\newcommand{\defref}[1]{Definition~\ref{def:#1}}

\newcommand{\trueerror}{{\text{\it err}}}

\begin{document}
\ifnum\full=0
\title{What Can We Learn Privately?\thanks{All omitted proofs and details appear in the full version~\cite{KLNRS08}.}}
\else
\title{What Can We Learn Privately?\thanks{A preliminary version of this paper appeared in 49th Annual IEEE Symposium on Foundations of Computer Science~\cite{KLNRSf08}.} }
\fi
\author{Shiva Prasad Kasiviswanathan\thanks{CCS-3, Los Alamos National Laboratory. Part of this work done while a student at Pennsylvania State University and supported by NSF award CCF-072891.} \and Homin K. Lee\thanks{Department of Computer Science, Columbia University, hkl7@columbia.edu} \and Kobbi Nissim\thanks{Department of Computer Science, Ben-Gurion University, kobbi@cs.bgu.ac.il. Supported in part by the Israel Science Foundation (grant 860/06), and by the Frankel Center for Computer Science.} \and Sofya Raskhodnikova\thanks{Department of Computer Science and Engineering, Pennsylvania State University, $\{$sofya,asmith$\}$@cse.psu.edu. S.R. and A.S. are supported in part by NSF award CCF-0729171.} \and Adam Smith\footnotemark[5]}

\maketitle

\begin{abstract}
Learning problems form an important category of computational tasks that generalizes many of the computations researchers apply to large real-life data sets. We ask: what concept classes can be learned privately, namely, by an algorithm whose output does not depend too heavily on any one input or specific training example? More precisely, we investigate  learning algorithms that satisfy {\em differential privacy}, a notion that provides strong confidentiality guarantees in contexts where aggregate information is released about a database containing sensitive information about individuals.
%shiva-previous line too long.
\ifnum\full=0
%shiva-added comma
We present several basic results that demonstrate general feasibility of private learning and relate several models
% ads 7/30/08: removed commas
previously studied separately in the contexts of privacy and standard learning.
% see http://owl.english.purdue.edu/handouts/grammar/g_comma.html	
% for info on using commas
%shiva - I added a comma before "and" bcos I thought the part before and after "and" are independent. But I really don't care. Thanks for the url.
\else

Our goal is a broad understanding of the resources required for private learning in terms of samples, computation time, and interaction. We demonstrate that, ignoring computational constraints, it is possible to privately agnostically learn any concept class using a sample size approximately logarithmic in the cardinality of the concept class. Therefore, almost anything learnable is learnable privately: specifically, if a concept class is learnable by a (non-private) algorithm with polynomial sample complexity and output size, then it can be learned privately using a polynomial number of samples. We also present a computationally efficient private PAC learner for the class of {\em parity} functions. This result dispels the similarity between learning with noise and private learning (both must be robust to small changes in inputs), since parity is thought to be very hard to learn given random classification noise.

{\em Local} (or {\em randomized response}) algorithms are a practical class of private algorithms that have received extensive investigation. We provide a precise characterization of local private learning algorithms. We show that a concept class is learnable by a local algorithm if and only if it is learnable in the {\em statistical query} (SQ) model. Therefore, for local private learning algorithms, the similarity to learning with noise is stronger: local learning is equivalent to SQ learning, and SQ algorithms include most known noise-tolerant learning algorithms.  Finally, we present a separation between the power of {\em interactive} and \emph{noninteractive} local learning algorithms. Because of the equivalence to SQ learning, this result also separates \emph{adaptive} and \emph{nonadaptive} SQ learning.
\fi
\end{abstract}

\section{Introduction}
The data privacy problem in modern databases is similar to that faced by  statistical agencies and medical researchers: to learn and publish global analyses of a population while maintaining the confidentiality of the participants in a survey. There is a vast body of work on this problem in statistics and computer science. However, until recently, most schemes proposed in the literature lacked rigorous analysis of privacy and utility.

A recent line of work%
\ifnum\full=1
~\cite{EGS03,DwNi04,BDMN05,DMNS06,DKMMN06,Dwork06,NRS07,DMT07,MT07,BCDKMT07,RHS07,BLR08,DY08}
\fi
, initiated by Dinur and Nissim~\cite{DiNi03} and called {\em private data analysis}, seeks to place data privacy on firmer theoretical foundations and has been  successful at formulating a strong, yet attainable privacy definition. The notion of {\em differential privacy}~\cite{DMNS06} that emerged from this line of work provides rigorous guarantees even in the presence of a malicious adversary with access to  arbitrary auxiliary information. It requires that whether an individual supplies her actual or fake information has almost no effect on the outcome of the analysis.

Given this definition, it is natural to ask: what computational tasks can be performed while maintaining privacy? Research on data privacy, to the extent that it formalizes precise goals,  has mostly focused on {\em function evaluation} (``what is the value of $f(\mathrm z)$?''), namely, how much privacy is possible if one wishes to release (an approximation to) a particular function $f$, evaluated on the database $\mathrm{z}$. (A notable exception is the recent work of McSherry and Talwar, using differential privacy in the design of auction mechanisms~\cite{MT07}). Our goal is to expand the utility of private protocols by examining which other computational tasks can be performed in a privacy-preserving manner.

\paragraph{Private Learning.}
\ifnum\full=1
Learning problems form an important category of computational tasks that generalizes many of the computations researchers apply to large real-life data sets.
\fi
In this work, we ask
what can be learned \emph{privately}, namely, by an algorithm whose output does not depend too heavily on any one input or specific training example.
% sofya: removed "?"
 Our goal is a broad understanding of the resources required for private learning in terms of samples, computation time, and interaction. We examine two basic notions from
\ifnum\full=1
computational learning theory:
\else
learning:
\fi
Valiant's probabilistically approximately correct (PAC) learning~\cite{Valiant84} model and Kearns' statistical query (SQ) model~\cite{Kearns98}.

Informally, a concept is a function from examples to labels, and a class of concepts is learnable if for any distribution $\DDD$ on examples, one can, given limited access to examples sampled from $\DDD$ labeled according to some target concept $c$, find a small circuit (hypothesis) which predicts $c$'s labels with high probability over
future examples taken from the same distribution. In the PAC model, a learning algorithm can access a polynomial number of labeled examples. In the SQ model, instead of accessing examples directly, the learner can specify some properties (i.e., predicates) on the examples, for which he is given an estimate, up to an additive polynomially small error, of the probability that a random example chosen from $\DDD $ satisfies the property. PAC learning is strictly stronger than the SQ learning ~\cite{Kearns98}.

We model a statistical database as a vector $\z =(z_1,\cdots,z_n)$, where each entry has been contributed by an individual. When analyzing how well a private algorithm learns a concept class, we assume that entries $z_i$ of the database are random examples generated i.i.d.\ from the underlying distribution $\DDD$ and labeled by a target concept $c$. This is exactly how (not necessarily private) learners are analyzed. For instance, an example might consist of an individual's gender, age, and blood pressure history, and the label, whether this individual has had a heart attack. The algorithm has to learn to predict whether an individual has had a heart attack, based on gender, age, and blood pressure history, generated according to $\DDD$.

We require a private algorithm to keep entire examples (not only the labels) confidential. In the scenario above, it translates to not revealing each participant's gender, age, blood pressure history, and heart attack incidence. More precisely, the output of a private learner should not be significantly affected if a particular example $z_i$ is replaced with arbitrary $z'_i$, for all $z_i$ and $z'_i$. In contrast to correctness or utility, which is analyzed with respect to distribution $\DDD$, differential privacy is a worst-case notion. Hence, when we analyze the privacy of our learners we do not make any assumptions on the underlying distribution. Such assumptions are fragile and, in particular, would fall apart in the presence of auxiliary knowledge
\ifnum\full=1
(also called background knowledge or side information)
\fi
that the adversary might have: conditioned on the adversary's auxiliary knowledge, the distribution over examples might look very different from $\DDD$. %shiva-explain this point in the full version.

\subsection{Our Contributions} \label{sec:contribs}
We introduce and formulate private learning problems, as discussed above, and develop novel algorithmic tools and bounds on the sample size required by private learning algorithms. Our results paint a picture of the classes of learning problems that are solvable subject to privacy constraints. Specifically, we provide:
\newcommand{\resref}[1]{{\bf (\ref{res:#1})}}
\begin{list}{{\bf (\arabic{enumi})}}{\usecounter{enumi}
\setlength{\leftmargin}{\parindent}
\setlength{\listparindent}{\parindent}
\setlength{\parsep}{-1pt}}

\item \label{res:occam} {\bf A Private Version of Occam's Razor.}  We
  present a generic private learning algorithm. For any concept class
  $\CCC$, we give a distribution-free differentially-private agnostic
  PAC learner for $\CCC$ that uses a number of samples proportional to
  $\log|\CCC|$. This is a private analogue of the ``cardinality
  version'' of \emph{Occam's razor}, a basic sample complexity bound
  from (non-private) learning theory. The sample complexity of our
  version is similar to that of the original, although the private
  algorithm is very different. As in Occam's razor, the learning
  algorithm is not necessarily computationally efficient.~\anote{Point
    to compression here?}

\item \label{res:parity} {\bf An Efficient Private Learner for Parity.}
We give a computationally efficient, distribution-free differentially private PAC learner for the class of {\em parity} functions\footnote{While the generic learning result \resref{occam} extends easily to ``agnostic'' learning (defined below), the learner for parity does not. The limitation is not surprising, since even non-private agnostic learning of parity is at least as hard as learning parity with random noise.} over $\zo^{d}$. The sample and time complexity are comparable to that of the best non-private learner.

\item \label{res:sq-local} {\bf Equivalence of Local (``Randomized Response'') and SQ Learning.}
We precisely characterize the power of {\em local}, or {\em randomized response}, private learning algorithms. Local algorithms are a special (practical) class of private algorithms and are popular in the data mining and statistics literature%
\ifnum\full=1
~\cite{W65,AS00,AA01,AH05,HH02,EGS03,MS06,HB08}%
\fi
. They add randomness to each individual's data independently before processing the input. We show that a concept class is learnable by a local differentially private algorithm if and only if it is learnable in the {\em statistical query} (SQ) model. This equivalence relates notions that were conceived in very different contexts.

\item \label{res:li-lni} {\bf Separation of Interactive and Noninteractive Local Learning.}
Local algorithms can be {\em noninteractive}, that is, using one round of interaction with individuals holding the data, or {\em interactive}, that is, using more than one round (and in each receiving randomized responses from individuals).
We construct a concept class, called {\em masked-parity}, that is efficiently learnable by {\em interactive}  local algorithms under the uniform distribution on examples, but requires an exponential (in the dimension) number of samples to be learned by a {\em noninteractive} local algorithm. The equivalence~\resref{sq-local} of local and SQ learning shows that interaction in local algorithms corresponds to {\em adaptivity} in SQ algorithms. The {\em masked-parity} class thus also separates adaptive and nonadaptive SQ learning.
\end{list}

\subsubsection{Implications}
\paragraph{``Anything'' learnable is privately learnable using few samples.} The generic agnostic learner \resref{occam} has an important consequence: if some concept class $\CCC$ is learnable by any algorithm, not necessarily a private one, whose output length in bits is polynomially bounded, then $\CCC$ is learnable privately using a polynomial number of samples (possibly in exponential time). This result establishes the basic feasibility of private learning: it was not clear {\em a priori} how severely privacy affects sample complexity, even ignoring computation time.

\paragraph{Learning with noise is different from private learning.}
There is an intuitively appealing similarity between learning from noisy examples and private learning: algorithms for both problems must be robust to small variations in the data. This apparent similarity is strengthened by a result of
Blum, Dwork, McSherry and Nissim~\cite{BDMN05} showing that any algorithm in Kearns' {\em statistical query} (SQ) model~\cite{Kearns98} can be implemented in a differentially private manner.
SQ was introduced to capture a class of noise-resistant learning algorithms. These algorithms access their input only through a sequence of approximate averaging queries. One can {\em privately} approximate the average of a function with values in $[0,1]$ over the data set of $n$ individuals to within additive error $O(1/n)$ (Dwork and Nissim~\cite{DwNi04}). Thus, one can simulate the behavior of an SQ algorithm privately, query by query.

Our  efficient private learner for parity \resref{parity}  dispels the similarity between learning with noise and private learning.
First, SQ algorithms provably require exponentially many (in the dimension)  queries to learn parity~\cite{Kearns98}. More compellingly, learning parity with noise is thought to be computationally hard, and has been used as the basis of several cryptographic primitives (\eg~\cite{BKW03,HB01,Alekhnovich03,Regev05parity}).

\paragraph{Limitations of local (``randomized response'') algorithms.}
%\hnote{Maybe we should move this paragraph to 1 or 1.1 and include the casual definition of interactive/noninteractive there.}\anote{Done, I think...}
 {\em Local} algorithms (also referred to as {\em randomized response}, {\em input perturbation}, {\em Post Randomization Method ({\small PRAM})}, and
{\em  Framework for High-Accuracy Strict-Privacy Preserving Mining ({\small FRAPP})})
have been studied extensively in the context of privacy-preserving data mining, both in statistics and computer science (\eg~\cite{W65,AS00,AA01,AH05,HH02,EGS03,MS06,HB08}).
Roughly, a local algorithm accesses each individual's data via independent randomization operators.
\ifnum\full=1
See \figref{models}, p.~\pageref{fig:models}.

\begin{figure*}
\vspace*{-2cm}
\begin{center}
\includegraphics[width=15cm]{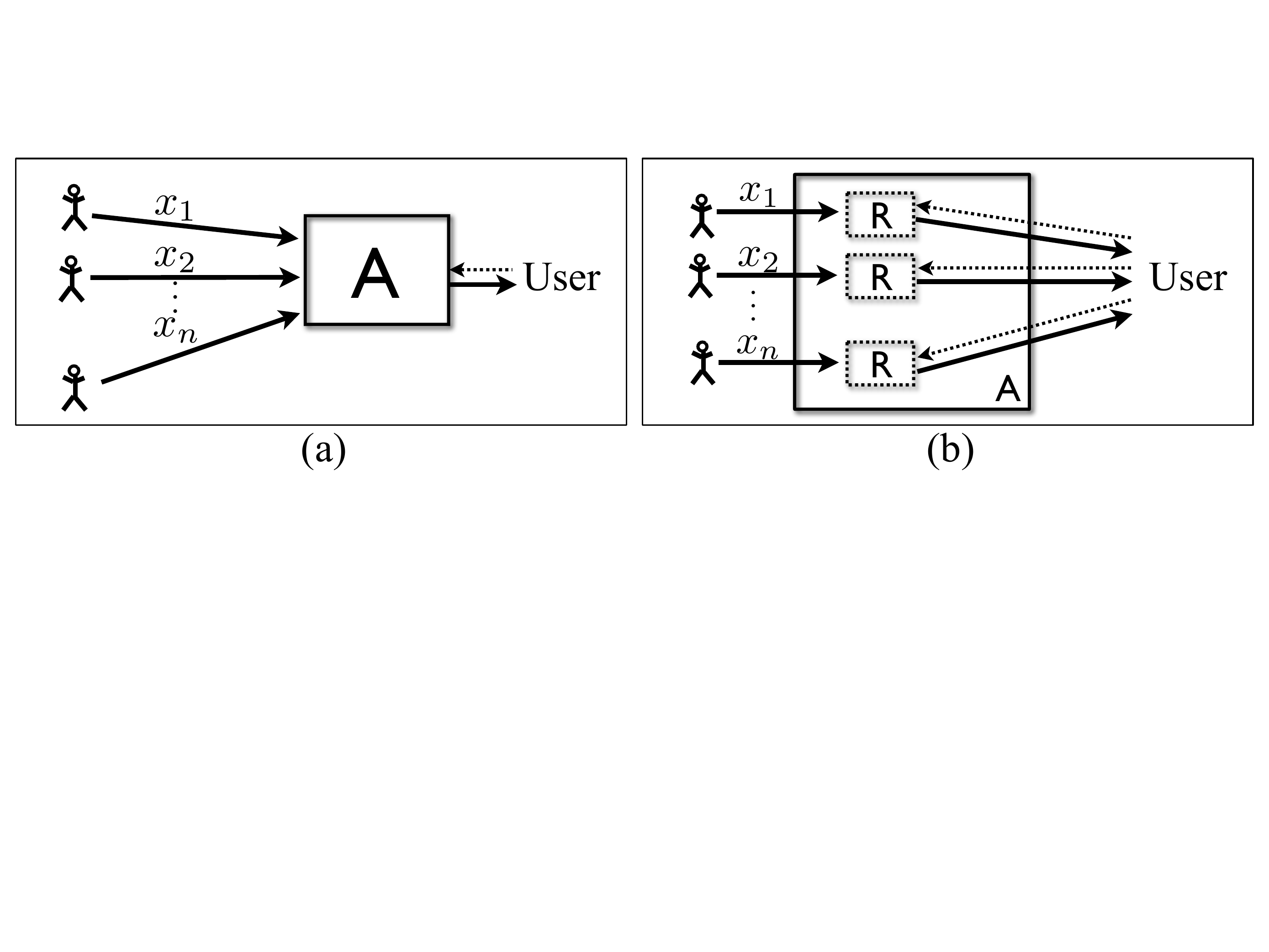}
\vspace*{-2.3in}
\caption{Two basic models for database privacy: (a) the \emph{centralized} model, in which data is collected by a trusted agency that publishes aggregate statistics or answers users' queries; (b) the \emph{local} model, in which users retain their data and run a randomization procedure locally to produce output which is safe for publication. The dotted arrows from users to data holders indicate that protocols may be completely noninteractive: in this case there is a single publication, without feedback from users.}
\label{fig:models}
\end{center}
\end{figure*}

\fi
Local algorithms were introduced to encourage truthfulness in surveys: respondents who know that their data will be randomized are more likely to answer honestly%
\ifnum\full=0
~(Warner \cite{W65}).
\else
. For example, Warner \cite{W65} famously considered a survey technique in which respondents are asked to give the correct answer to a sensitive (true/false) question with probability $2/3$ and the incorrect answer with probability $1/3$, in the hopes that the added uncertainty would encourage them to answer honestly. The proportion of ``true'' answers in the population is then estimated using a standard, non-private deconvolution.
\fi
The accepted privacy requirement for local algorithms is equivalent to imposing differential privacy on each randomization operator~\cite{EGS03}. Local algorithms are popular because they are easy to understand and implement. In the extreme case, users can retain their data and apply the randomization operator themselves, using %either
a physical device~\cite{W65,MN06} or a cryptographic protocol~\cite{AJL04}.

The equivalence between local and SQ algorithms \resref{sq-local} is a powerful tool that allows us to apply results from learning theory. In particular, since parity is not learnable with a small number of SQ queries~\cite{Kearns98} but is PAC learnable
%sofya: removed ref to centralized model
%in the {\em centralized} (i.e., non-local) model,
privately \resref{parity},
we get that
local algorithms require exponentially more data for some learning tasks than do general private algorithms. Our results also imply that local algorithms are strictly {\em less} powerful than (non-private) algorithms for learning with classification noise because subexponential (non-private) algorithms can learn parity with noise~\cite{BKW03}.

\paragraph{Adaptivity in SQ algorithms is important.}
%sofya: changed the wording
%ads 8/3/08 interaction in local algs now explained in bullet (3).
%Local algorithms can be {\em interactive}, meaning that the randomization operators they apply might depend on previous responses. Similarly,
Just as local algorithms can be interactive,
SQ algorithms can be {\em adaptive}, that is,
the averaging queries they make may depend on answers to previous queries.
The equivalence of SQ and local algorithms \resref{sq-local} preserves interaction/adaptivity: a concept class is nonadaptively SQ learnable if and only if it is noninteractively locally learnable. The  masked parity class \resref{li-lni} shows that interaction (resp., adaptivity) adds considerable power to local (resp., SQ) algorithms.

Most of the reasons that local algorithms are so attractive in practice, and have received such attention, apply only to noninteractive algorithms (interaction can be costly, complicated, or even impossible---for instance, when statistical information is collected by an interviewer, or at a polling booth).

This suggests that further investigating the power of nonadaptive SQ learners is an important problem. For example, the SQ algorithm for learning conjunctions~\cite{Kearns94} is nonadaptive, but SQ formulations of the perceptron and $k$-means algorithms~\cite{BDMN05} seem to rely heavily on adaptivity.

%A proof that practical computations such as these  cannot be efficiently implemented in the nonadaptive SQ model would have significant repercussions for privacy-preserving data mining.
%\anote{I only thought of this last night... does anyone disagree? Is it better just not to state this as a question until we think about it more?}\rnote{What are we gaining by advertising this open question?}

\paragraph{Understanding the ``price'' of privacy for learning problems.} The SQ result of Blum \etal\cite{BDMN05} and our learner for parity \resref{parity}  provide efficient (i.e., polynomial time) private learners for essentially all the concept classes known (by us) to have efficient non-private distribution-free learners. Finding a concept class that can be learned efficiently, but not privately and efficiently, remains an interesting and important question.

Our results also lead to questions of optimal sample complexity for learning problems of practical importance. The private simulation of SQ algorithms
\ifnum\full=1
due to Blum \etal%
\else
in
\fi
\cite{BDMN05} uses a factor of approximately $\sqrt{t}/\eps$ more data points than the na{\"\i}ve non-private implementation, where $t$ is the number of SQ queries and $\eps$ is the parameter of differential privacy (typically a small constant). In contrast, the generic agnostic learner \resref{occam} uses a factor of at most $1/\eps$ more samples than the corresponding non-private learner. For parity, our private learner uses a factor of roughly $1/\eps$ more samples than, and about the same computation time as, the non-private learner.  What, then, is the additional cost of privacy when  learning practical concept classes (half-planes, low-dimensional curves, etc)? Can the theoretical sample bounds of \resref{occam} be matched by (more) efficient learners?

\subsubsection{Techniques}\label{sec:techniques}
Our generic private learner \resref{occam} adapts the exponential sampling technique of
McSherry and Talwar
~\cite{MT07}, developed in the context of auction design. Our use of the exponential mechanism inspired an elegant {\em subsequent} result of Blum, Liggett, and Roth~\cite{BLR08} (BLR) on simultaneously approximating many different functions.
\anote{Deleted some discussion here made irrelevant by simpler proof of VC bound.}
% Their result can, in turn, be used to derive a version of our statement for hypotheses of bounded VC dimension (rather than bounded cardinality), when the space of examples has bounded size%
% \ifnum\full=1
% \ (see \secref{vclearn})%
% \fi
% . The generic private learner from this paper and that implied by the BLR result are incomparable: roughly, our original result requires discretizing (quantizing) the set of hypotheses, whereas the BLR result requires discretizing the space of examples.  Neither
% achieves the generality of the original Vapnik-Chernovenkis bound (see~\cite{Kearns94}), which requires only bounded VC dimension and makes no  assumptions on the cardinality of either the hypothesis or example space.

The efficient private learner for parity \resref{parity} uses a very different technique, based on sampling, running a non-private learner, and occasionally refusing to answer based on delicately calibrated probabilities. Running a non-private learner on a random subset of examples is a very intuitive approach to building private algorithms, but it is not private in general. The private learner for parity illustrates both why this technique can leak private information and how it can sometimes be repaired based on special (in this case, algebraic) structure.

The interesting direction of the equivalence between SQ and local learners \resref{sq-local} is proved via a simulation of any local algorithm by a corresponding SQ algorithm. We found this simulation surprising since local protocols can, in general, have very complex structure (see, \eg~\cite{EGS03}). The SQ algorithm proceeds by a direct simulation of the output of the randomization operators. For a given input distribution $\DDD$ and any operator $R$, one can sample from the corresponding output distribution $R(\DDD)$ via rejection sampling.
We show that if $R$ is differentially private, the rejection probabilities can be approximated via low-accuracy SQ queries to $\DDD$.

Finally, the separation between adaptive and nonadaptive SQ \resref{li-lni} uses a Fourier analytic argument inspired by Kearns' SQ lower bound for parity~\cite{Kearns98}.

\ifnum\full=1
\subsubsection{Classes of Private Learning Algorithms}
\newcommand{\ineff}[1]{\ensuremath{{#1}^*}}
\newcommand{\ipac}{\ineff{\pac}}
\newcommand{\isq}{\ineff{\sq}}
\newcommand{\ippac}{\ineff{\ppac}}
\newcommand{\inasq}{\ineff{\nasq}}
\newcommand{\ilni}{\ineff{\lni}}
\newcommand{\ili}{\ineff{\li}}

\begin{figure}[ht]
\begin{center}
\includegraphics[height=1.3in]{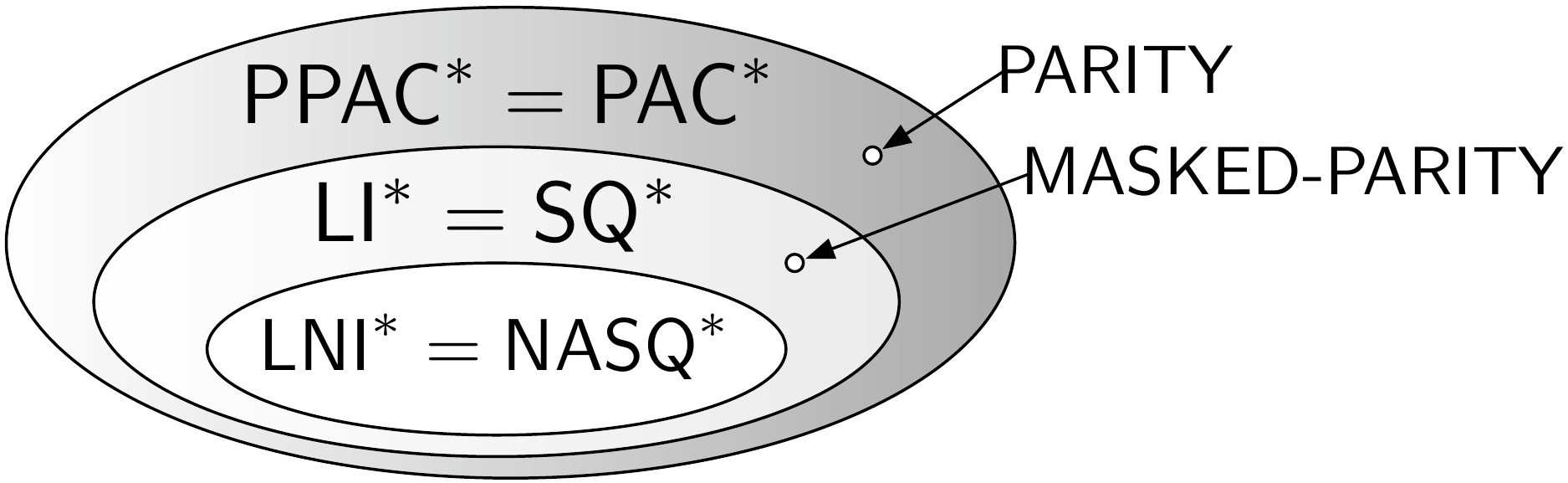}
\caption{\footnotesize Relationships  among learning classes taking into account sample complexity, but not computational efficiency.}
\label{fig:classes}
\end{center}
\end{figure}

We can summarize our results via a complexity-theoretic picture of learnable and privately learnable concept classes (more precisely, the members of the classes are pairs of concept classes and example distributions). In order to make asymptotic statements, we measure complexity in terms of the length $d$ of the binary description of examples.

We first consider learners that use a polynomial (in $d$) number of samples and output a hypothesis that is described using a polynomial number of bits, but have unlimited computation time. Let $\ipac$ denote the set of concept classes that are learnable by such algorithms ignoring privacy, and let $\ippac$ denote the subset of $\ipac$ learnable by differentially private\footnote{Differential privacy is quantified by a real parameter $\eps>0$. To make qualitative statements, we look at algorithms where $\eps\to0$ as $d\to\infty$. Taking $\eps=1/d^c$ for any constant $c>0$ would yield the same class.} algorithms.

Since we restrict the learner's output to a polynomial number of bits, the hypothesis classes of the algorithms are {\em de facto} limited to have size at most $\exp(poly(d))$.  Thus, the generic private learner (point \resref{occam} in the introduction) will use a polynomial number of samples, and  $\ipac=\ippac$.

We can similarly interpret the other results above. Within $\ipac$,  we can consider subsets of concepts learnable by SQ algorithms (\isq), nonadaptive SQ algorithms ($\inasq$),  local interactive algorithms ($\ili$) and  local noninteractive algorithms ($\ilni$). We obtain the following picture (see page~\pageref{fig:classes}):
$$\ilni=\inasq \ \ \ \subsetneq \ \ \ \ili=\isq \ \ \  \subsetneq \ \ \ \ippac=\ipac.
$$
The equality of $\ili$ and $\isq$, and of $\ilni$ and $\inasq$, follow from the SQ simulation of local algorithms (Theorem~\ref{thm:nasq-nilr}). The parity and masked-parity concept classes separate $\ippac$ from $\isq$ and $\isq$ from $\inasq$, respectively (Corollaries~\ref{cor:sep1} and~\ref{cor:sep2}). (\emph{Note:} The separation of $\ippac$ from $\isq$ holds even for distribution-free learning; in contrast, the separation of $\isq$ from $\inasq$ holds for learnability under a specific distribution on examples, since the adaptive SQ learner for $\mparity$ requires a uniform distribution on examples.)

When we take computational efficiency into account, the picture changes. The relation between local and SQ classes remain the same modulo a technical restriction on the randomization operators (\defref{transparent}). $\sq$ remains distinct from $\ppac$ since parity is efficiently learnable privately. However, it is an open question whether concept classes which can be efficiently learned can also be efficiently learned privately.

%Since known examples of efficiently learnable concept classes either lie in SQ or are closely related to parity, even finding a plausible candidate class where privacy changes the efficiency of learning remains an open question.
%We can summarize our results via a complexity-theoretic picture of learnable and privately learnable concept classes. The diagram at right presents the relationships between classes of learning algorithms, taking into account polynomial sample complexity and output size but ignoring computational efficiency. See \appref{classes} for definitions and discussion.\anote{Too brief???}
\fi

\subsection{Related Work}\label{sec:relwork}

Prior to this work, the
literature on differential privacy studied function approximation tasks (e.g.~~\cite{DiNi03,DwNi04,BDMN05,DMNS06,NRS07,BCDKMT07}), with the exception of
the work of McSherry and Talwar on mechanism design~\cite{MT07}.
%
% ads 3/5/09: This sentence was distractng from the main point
%Indeed, a wealth of techniques for private function evaluation in the centralized model have been presented in the last few years, where individual privacy is protected by the addition of low magnitude noise, carefully crafted to hide individual influence on the analysis' result~\cite{DiNi03,DwNi04,BDMN05,DMNS06,NRS07,BCDKMT07}.
%
Nevertheless, several of these  prior results have direct implications to machine learning-related problems.
Blum~\etal\cite{BDMN05} considered a particular class of learning algorithms (SQ), and showed that algorithms in the class could be simulated using noisy function evaluations. In an independent, unpublished work, Chaudhuri, Dwork, and Talwar considered a version of private learning in which privacy is afforded only to input labels, but not to examples.
Other works considered specific machine learning problems such as mining frequent itemsets~\cite{EGS03}, $k$-means clustering~\cite{BDMN05,NRS07}, learning decision trees~\cite{BDMN05}, and learning mixtures of Gaussians~\cite{NRS07}.

As mentioned above, a subsequent result of   Blum, Ligett and Roth~\cite{BLR08} on approximating classes of low-VC-dimension functions was inspired by our generic agnostic learner. We discuss their result further in Section~\ref{sec:pac-ppac}. Since the original version of our work, there have also been several results connecting differential privacy to more ``statistical'' notions of utility, such as consistency of point estimation and density estimation~\cite{Smith08,DL09,WZ08,ZLW09}.

Our separation of interactive and noninteractive protocols in the {\em local} model \resref{sq-local}  also has a precedent:  Dwork \etal\cite{DMNS06} separated interactive and noninteractive private protocols in the {\em centralized} model, where the user accesses the data via a server that runs differentially private algorithms on the database and sends back the answers. That separation has a very different flavor from the one in this work: any example of a computation that cannot be performed noninteractively in the centralized model must rely on the fact that the computational task is not defined until after the first answer from the server is received.
(Otherwise, the user can send an algorithm for that task to the server holding the data, thus obviating the need for interaction.) In contrast, we present a computational task that is hard for noninteractive local algorithms -- learning {\em masked parity} -- yet is defined in advance.
%our separation of interactive and noninteractive local protocols \resref{sq-local} is of a different nature: the computational task that is hard for noninteractive local algorithms -- learning {\em masked parity} -- is defined in advance.

In the machine learning literature, several notions similar to differential privacy have been explored under the rubric of ``algorithmic stability'' \cite{DevroyeWagner-1979,KearnsRon-1999,BousquetElisseeff-2002,KN02,EliEvgPon-2005,BPS07}. The most closely related notion is  \emph{change-one error stability}, which measures how much the generalization error changes when an input is changed (see the survey \cite{KN02}). In contrast, differential privacy measures how the distribution over the entire output changes---a more complex measure of stability (in particular, differential privacy implies change-one error stability).
A different notion, stability under resampling of the data from a given distribution~\cite{BLP06,BPS07}, is connected to the sample-and-aggregate method of \cite{NRS07} but is not directly relevant to the techniques considered here.
\ifnum\full=1
Finally, in a different vein, Freund, Mansour and Schapire ~\cite{FMS04} used a weighted averaging technique with the same weights as the sampler in our generic learner to reduce generalization error (see \secref{pac-ppac}).
\fi

\section{Preliminaries} \label{sec:defs}
%shiva-added this para
%sofya: rephrased to make shorter. I do not think most of it is necessary in the short version.
We use $[n]$ to denote the set $\{1,2,\dots,n\}$. Logarithms base $2$ and base $e$ are denoted by $\log$ and $\ln$, respectively.
\ifnum\full=1
$\Pr[\cdot]$ and $\E[\cdot]$ denote probability and expectation, respectively.
\fi
$\mathcal{A}(x)$ is the probability distribution over outputs of a randomized algorithm $\mathcal{A}$ on input $x$.  The {\em statistical difference} between
%probability measures
distributions $\erert$ and $\Q$ on a discrete space $D$ is defined as $\max_{S \subset D}|\erert(S)-\Q(S)|$.
\subsection{Differential Privacy}%Preliminaries from Databases and Privacy}

A statistical database is a vector $\z=(z_1,\dots,z_n)$  over a domain $D$, where each entry $z_i\in D$ represents information contributed by one individual. Databases $\z$  and $\mathrm{z'}$ are {\em neighbors} if $z_i\neq z'_i$ for exactly one $i\in[n]$ (i.e., the Hamming distance between $\z$ and $\mathrm{z'}$ is 1).
\ifnum\full=1
%sofya: Do we need this?
All our algorithms are symmetric, that is, they do not depend on the order of entries in the database $\z$. Thus, we could define a database as a multi-set in $D$, and use symmetric difference instead of the Hamming metric to measure distance. We adhere to the vector formulation for consistency with the previous works.
\fi

A (randomized) algorithm (in our context, this will usually be a learning algorithm) is private if neighboring databases induce nearby distributions on its outcomes:

\begin{definition}[$\epsilon$-differential privacy~\cite{DMNS06}] \label{def:eps-dp} A randomized algorithm $\mathcal{A}$ is $\epsilon$-differentially private if for all neighboring databases $\z ,\mathrm{z'}$, and for all sets $\mathcal{S}$ of %possible
outputs,
\ifnum\full=0
$\Pr[\mathcal{A}(\z ) \in \mathcal{S}] \leq \exp(\epsilon) \cdot \Pr[\mathcal{A}(\mathrm{z'}) \in \mathcal{S}]$.
\else
\begin{eqnarray*} &\Pr[\mathcal{A}(\z ) \in \mathcal{S}] \leq \exp(\epsilon) \cdot \Pr[\mathcal{A}(\mathrm{z'}) \in \mathcal{S}].  & \end{eqnarray*}
\fi
The probability is taken over the random coins of $\mathcal{A}$.
\end{definition}

\ifnum\full=1 In ~\cite{DMNS06}, the notion above was called
``indistinguishability''. The name ``differential privacy'' was
suggested by Mike Schroeder, and first appeared in
Dwork~\cite{Dwork06}.

Differential privacy composes well (see, \eg
\cite{DKMMN06,NRS07,MT07,KS08}):
\begin{claim}[Composition and Post-processing]\label{claim:composition}
\else
\begin{claim}[Composition and Post-processing~\cite{DKMMN06,NRS07,MT07,KS08}]\label{claim:composition}
\fi
If a randomized algorithm $\AAA$ runs $k$ algorithms $\AAA_{1},...,\AAA_{k}$, where each $\AAA_i$ is
$\eps_i$-differentially private, and outputs a function of the results (that is, $\AAA(\z) = g(\AAA_{1}(\z),\linebreak\AAA_{2}(\z),...,\AAA_{k}(\z))$ for some probabilistic algorithm $g$),  then $\AAA$ is $(\sum_{i=1}^{k}\eps_i)$-differentially private.
\end{claim}

One method for obtaining efficient differentially private algorithms for approximating real-valued functions is based on adding Laplacian noise to the true answer. %sofya: need to cite all papers that do that
Let $\Lap(\lambda)$ denote the Laplace probability distribution with mean $0$, standard deviation $\sqrt{2}\lambda$, and p.d.f.\ $f(x)=\frac{1}{2\lambda}e^{-|x|/\lambda}$.
\begin{theorem} [Dwork \etal\cite{DMNS06}]\label{thm:DMNS}
For a function $f: D^n \rightarrow \R$, define its global sensitivity $GS_f = \max_{\mathrm{z},\mathrm{z'}}|f(\mathrm{z})-f(\mathrm{z'})|$ where the maximum is over all neighboring databases $\mathrm{z},\mathrm{z'}$. Then, an algorithm that on input $\mathrm{z}$ returns $f(\mathrm{z})+\eta$ where $\eta \sim \Lap(GS_f/\eps)$ is $\eps$-differentially private.
\end{theorem}

\subsection{Preliminaries from Learning Theory} \label{sec:ppac}
A concept is a function that labels {\em examples} taken from the domain $X$ by the elements of the range~$Y$.
\ifnum\full=0 %ads 8/1/08
We focus on binary classification problems, where the range $Y$ is $\{0,1\}$ (or, equivalently,  $\{+1,-1\})$.
\fi
A~\emph{concept class} $\CCC$ is a set of concepts. It comes implicitly with a way to represent concepts;
$\mathrm{\it size}(c)$ is the size of the (smallest) representation of $c$ under the given representation scheme.
%Let $\CCC$ be a concept class over domain $X$.
% ads 8/1/08
\ifnum\full=0
The domain of the concepts in $\CCC$ is understood to be an ensemble $X= \{X_d\}_{d\in\N}$ %and $Y=\{Y_d\}_{d\in\N},$
where the representation of elements in $X_d$ %,Y_d$
is of size at most $d$.
\else
The domain and the range of the concepts in $\CCC$ are understood to be ensembles $X= \{X_d\}_{d\in\N}$ and $Y=\{Y_d\}_{d\in\N},$ where the representation of elements in $X_d,Y_d$
is of size at most $d$.
%shiva-changed Y_d from {0,1} to {-1,1} for consistency.
We focus on binary classification problems, in which the label space $Y_d$ is $\{0,1\}$ or $\{+1,-1\}$; the parameter $d$ thus measures the size of the examples in $X_d$.
\fi
(We use the parameter $d$ to formulate asymptotic complexity notions.) %, \eg polynomial time and sample complexity; however, all of our results have concrete analogues. %shiva-what analogues.
The concept classes are  ensembles $\CCC = \{\CCC_d\}_{d\in\N}$ where $\CCC_d$ is the class of concepts from $X_d$ to
%ads 8/1/08
\ifnum\full=1 $Y_d$.  \else $\{0,1\}$. \fi
%
%ads 8/1/08
%When the size parameter is clear from the context or not important, we write $c:X\rightarrow Y$ instead of $c:X_d\rightarrow Y_d$ for $c\in\CCC_d$.
When the size parameter is clear from the context or not important, we omit the subscript in
\ifnum\full=1 $X_d,Y_d,\CCC_d$. \else $X_d,\CCC_d$. \fi

Let $\DDD$ be a distribution over labeled examples in
\ifnum\full=1 $X_d\times Y_d$. \else $X_d\times \{0,1\}$. \fi
A learning algorithm is given access to $\DDD$ (the method for accessing $\DDD$ depends on the type of learning algorithm). It outputs a hypothesis $h: X_d \to
\ifnum\full=1 Y_d$ \else \{0,1\}$ \fi
from a hypothesis class $\HHH=\{\HHH_d\}_{d\in\N}$. The goal is to minimize the {\em misclassification error} of $h$ on $\DDD$, defined as \ifnum\full=0 $\trueerror(h)=\Pr_{(x,y) \sim \DDD} [h(x)\neq y]\,.$ \else \[\trueerror(h)=\Pr_{(x,y) \sim \DDD} [h(x)\neq y]\,.\] \fi
The success of a learning algorithm is quantified by parameters $\alpha$ and $\beta$, where $\alpha$ is the desired error and $\beta$ bounds the probability of failure to output a hypothesis with this error.
\ifnum\full=1
Error measures other than misclassification are considered in supervised learning (\eg $L_2^2$). We study only misclassification error here, since for binary labels it is equivalent to the other common error measures.
\fi

 A learning algorithm is usually given access to an oracle that produces i.i.d.\ samples from $\DDD$.  Equivalently, one can view the learning algorithm's input as a list of $n$ labeled examples, i.e., $\z\in D^n$ where $D=X_d\times
\ifnum\full=1 Y_d$. \else \{0,1\}$. \fi
\ifnum\full=1
PAC learning and agnostic learning are described in Definitions~\ref{def:PAC} and~\ref{def:agnostic}. Another common method of access to $\DDD$ is via ``statistical queries'', which return the approximate average of a function over the distribution. Algorithms that work in this model can be simulated given i.i.d.\ examples. See \secref{sqlocal}. %
\fi

PAC learning algorithms are frequently designed assuming a promise that the examples are labeled consistently with some {\em target} concept $c$ from a class $\CCC$: namely, $ c \in \CCC_d$ and $y=c(x)$ for all $(x,y)$ in the support of $\DDD$.
%shiva -old line changed below
%In that case, we can think of $\DDD$ as a distribution only over examples $X_d$ and the error of hypothesis $h$ becomes $\trueerror(h)=\Pr_{x \sim \DDD}[ h(x)\neq c(x)].$
In that case, we can think of $\DDD$ as a distribution only over examples $X_d$. To avoid ambiguity, we use $\XXX$ to denote a distribution over $X_d$. In the PAC setting, $\trueerror(h)=\Pr_{x \sim \XXX}[ h(x)\neq c(x)].$

\begin{definition}[PAC Learning] \label{def:PAC}
%Let $\CCC$ be a concept class over $X$.
A concept class $\CCC$ over $X$ is {\em PAC learnable using hypothesis class $\HHH$} if there exist an algorithm $\mathcal{A}$ and a polynomial $poly(\cdot,\cdot,\cdot)$ such that for all $d \in \N$, all concepts $c \in \CCC_d$, all distributions $\XXX$ on $X_d$, and all $\alpha,\beta \in (0,1/2)$, given inputs $\alpha, \beta$ and $\z =(z_1,\cdots,z_n),$ where $n=poly(d,1/\alpha,\log (1/\beta))$, $z_i=(x_i,c(x_i))$ and $x_i$ are drawn i.i.d.\ from $\XXX$ for $i\in[n]$, algorithm $\AAA$ outputs a hypothesis $h\in \HHH$ satisfying
\begin{eqnarray}\label{eqn:pac-error}
\Pr[\trueerror(h) \leq \alpha] \geq 1-\beta.
\end{eqnarray}
The probability is taken over the random choice of the examples $\z $ and the coin tosses of $\mathcal{A}$.

Class $\CCC$ is (inefficiently) \emph{PAC learnable} if there exists some hypothesis class $\mathcal{H}$ and a PAC learner $\AAA$ such that $\AAA$ PAC learns $\CCC$ using $\HHH$. Class $\CCC$ is {\em efficiently} PAC learnable if $\AAA$ runs it time polynomial in $d,1/\alpha$, and $\log (1/\beta).$
\end{definition}

\begin{remark}
Our definition deviates slightly from the standard one (see, \eg \cite{Kearns94}) in that we do not take into consideration the size of the concept $c$. This choice allows us to treat PAC learners and agnostic learners identically.
% ads 8/1/08
\ifnum\full=1
One can change Definition~\ref{def:PAC} so that the number of samples depends polynomially also on the size of $c$ without affecting any of our results significantly.
\fi
\end{remark}

Agnostic  learning~\cite{H92,KSS92} is an extension of PAC learning that removes assumptions about the target concept.  \ifnum\full=1
Roughly speaking, the goal of an agnostic learner for a concept class $\CCC$ is to output a hypothesis $h\in\HHH$ whose error with respect to the distribution is close to the optimal possible by a function from $\CCC$.
In the agnostic setting,
\else
In this setting,
\fi
$\trueerror(h) = \Pr_{(x,y) \sim \DDD}[h(x) \neq y]$.

\begin{definition}[Agnostic Learning]
\label{def:agnostic}
{\em (Efficiently) agnostically  learnable} is defined identically to (efficiently) PAC learnable with two exceptions: {\em(i)} the data are drawn from an arbitrary distribution $\DDD$ on \ifnum\full=1 $X_d\times Y_d$;  \else $X_d\times \{0,1\}$; \fi
{\em(ii)} instead of Equation~\ref{eqn:pac-error}, the output of $\AAA$ has to satisfy:
$$\Pr[\trueerror(h) \leq OPT+\alpha] \geq 1-\beta,$$
where  $OPT=\min_{f \in \CCC_d}\left\{\trueerror(f)\right\}.$
\ifnum\full=1
As before, the probability is taken over the random choice of $\mathrm{z}$, and the coin tosses of $\mathcal{A}$.
\fi
\end{definition}
%shiva-old line changed below
%Definitions~\ref{def:PAC} and~\ref{def:agnostic} capture {\em distribution-free} learning, in that they do not assume a particular form for the distribution $\DDD$ (beyond the consistency promise for non-agnostic learning).
Definitions~\ref{def:PAC} and~\ref{def:agnostic} capture {\em distribution-free} learning, in that they do not assume a particular form for the distributions $\XXX$ or $\DDD$. In \secref{mparity}, we also consider learning algorithms that assume  a {\em specific} distribution $\DDD$ on examples (but make no assumption on which concept in $\CCC$ labels the examples). When we discuss such algorithms, we specify  $\DDD$ explicitly; without qualification, ``learning'' refers to distribution-free learning.

\ifnum\full=1
\paragraph{Efficiency Measures.} The definitions above are sufficiently detailed to allow for exact complexity statements (\eg ``$\AAA$ learns $\CCC$ using $n(\alpha,\beta)$ examples and time $O(t)$''), and the upper and lower bounds in this paper are all stated in this language. However, we also focus on two broader measures to allow for qualitative statements: {\em (a)} {\em polynomial sample complexity} is the default notion in our definitions. With the novel restriction of privacy, it is not {\em a priori} clear which concept classes can be learned using few examples even if we ignore computation time. {\em (b)} We use the term {\em efficient} private learning to impose the additional restriction of {\em polynomial computation time} (which implies polynomial sample complexity).
\fi

\section{Private PAC and Agnostic Learning}
We define private PAC learners as algorithms that satisfy definitions of both differential privacy and PAC learning. We emphasize that these are qualitatively different requirements.  Learning must succeed on average over a set of examples drawn i.i.d.\ from $\DDD$ (often under the additional promise that $\DDD$ is consistent with a concept from a target class). Differential privacy,  in contrast, must hold in the worst case, with no assumptions on consistency.

\begin{definition}[Private PAC Learning]\label{def:private-general}
%shiva-8/3/08: eps \in (0,1] should be removed.
%ads 8/3/08: removed
Let $d,\alpha,\beta$ be as in Definition~\ref{def:PAC} and $\eps>0$.  Concept class $\CCC$ is (inefficiently) \emph{privately PAC learnable} using hypothesis class $\HHH$ if there exists an algorithm $\mathcal{A}$ that takes inputs $\eps,\alpha,\beta,\mathrm{z}$, where $n$, the number of labeled examples in $\mathrm{z}$, is polynomial in $1/\eps,d,1/\alpha,\log(1/\beta)$, and satisfies
%sofya: changed order for consistency with Parity statement and to save a line
\renewcommand{\labelenumi}{\alph{enumi}.}
\begin{enumerate}
\item {\ensuremath{\sf [Privacy]}} For all $\eps>0$, algorithm $\mathcal{A(\eps,\cdot,\cdot,\cdot)}$ is $\epsilon$-differentially private (Definition~\ref{def:eps-dp});
\item {\ensuremath{\sf [Utility]}}
\ifnum\full=1
Algorithm
\fi
$\mathcal{A}$ PAC learns $\CCC$ using $\HHH$ (Definition~\ref{def:PAC}).
\end{enumerate}
$\CCC$ is {\em efficiently} privately PAC learnable if $\AAA$ runs in time polynomial in $d,1/\eps,1/\alpha$, and $\log (1/\beta).$
\end{definition}

\begin{definition}[Private Agnostic Learning]\label{def:private-agnostic}
(Efficient) {\em private agnostic learning} is defined analogously to (efficient) private PAC learning with Definition~\ref{def:agnostic} replacing Definition~\ref{def:PAC} in the utility condition.
\end{definition}
Evaluating the quality of a particular hypothesis is easy: one can privately compute the fraction of the data it classifies correctly (enabling cross-validation) using the sum query framework of~\cite{BDMN05}. The difficulty of constructing private learners lies in finding a good hypothesis in what is typically an exponentially large space.

\subsection{A Generic Private Agnostic Learner} \label{sec:pac-ppac}
In this section, we present a private analogue of a basic consistent learning result, often called the cardinality version of Occam's razor\footnote{We discuss the relationship to the ``compression version'' of Occam's razor at the end of this section.}. This classical result shows that a PAC learner can weed out all {\em bad} hypotheses given a number of labeled examples that is logarithmic in the size of the hypothesis class (see~\cite[p.~35]{Kearns94}). Our generic private learner is based on the exponential mechanism of McSherry and Talwar \cite{MT07}.

Let $q: D^n \times \HHH_d \rightarrow \R$  take a database $\z$ and a candidate hypothesis $h$, and assign it a score
$q(\mathrm{z},h) = -|\{i \,: \, x_i \mbox{ is misclassified by } h, \mbox{ \ie } y_i \neq h(x_i) \}|\,.$ That is, the score is minus the number of points in $\z$ misclassified by $h$. The classic Occam's razor argument assumes a learner that selects a hypothesis with maximum score (that is, minimum empirical error). Instead, our private learner $\mathcal{A}^{\epsilon}_q$ is defined to sample a random hypothesis with probability dependent on its score:
\ifnum\full=0
\begin{eqnarray*}
\mathcal{A}^{\epsilon}_q(\z):& \text{Output hypothesis $h\in \HHH_d $ with probability}\\
&\text{proportional to $\exp \left (\frac 1 2{\epsilon q(\z,h)}\right )$ }.
\end{eqnarray*}
\else
$$\mathcal{A}^{\epsilon}_q(\z) \quad : \quad  \mbox {Output hypothesis $h\in \HHH_d $ with probability proportional to $\exp \left (\frac{\epsilon q(\z,h)}{2}\right )$\,}.$$
\fi
Since the score ranges from $-n$ to 0, hypotheses with low empirical error are exponentially more likely to be selected than ones with high error.

Algorithm $\mathcal{A}^{\epsilon}_q$
fits the framework of McSherry and Talwar, and so is $\eps$-differentially private. This follows from the fact that changing one entry $z_i$ in the database $\z$ can change the score by at most 1.
\begin{lemma}[following \cite{MT07}] \label{lem:mtt}
The algorithm $\mathcal{A}^{\epsilon}_q$ is $\epsilon$-differentially private.
\end{lemma}

A similar exponential weighting algorithm was considered by Freund, Mansour and Schapire~\cite{FMS04} for constructing binary classifiers with good generalization error bounds. We are not aware of any direct connection between the two results. Also note that, except for the case where $|\HHH_d|$ is  polynomial, the exponential mechanism $\mathcal{A}^{\epsilon}_q(\mathrm{z})$ does not necessarily yield a polynomial time algorithm.
%sofya: commented out -- see last sent. in the second paragraph of sec 3.1
%\ifnum\full=1 %shiva-added this line for the full version.
%The idea is that under the distribution $\mathcal{A}^{\epsilon}_q$ samples from, the good hypotheses (ones that %perform well on the training data) appear with exponential more probability than the bad ones.
%\fi

\begin{theorem}[Generic Private Learner]\label{thm:PACvsPPAC} For all $d \in \N$, any concept class $\CCC_d$ whose cardinality is at most $\exp({\poly(d)})$ is privately agnostically learnable using $\HHH_d=\CCC_d$. More precisely, the learner uses $n=O((\ln |\HHH_d  | +\ln \frac 1 {\beta})\cdot \max\{\frac{1}{\epsilon\alpha},\frac{1}{\alpha^2}\})$ labeled examples from $\DDD$, where $\eps, \alpha$, and $\beta$ are parameters of the private learner. (The learner might not be efficient.)
\end{theorem}
\begin{proof}
Let $\mathcal{A}^{\epsilon}_q$ be as defined above. The privacy condition in Definition~\ref{def:private-general} is satisfied by \lemref{mtt}.

We now show that the utility condition is also satisfied. Consider the event
$E=\{\mathcal{A}^{\epsilon}_q (\mathrm{z})= h \;\mathrm{ with }\ \trueerror(h) > \alpha + OPT\}.$
We want to prove that $\Pr[E] \leq \beta$. Define the training error of $h$ as $$\trueerror_T(h) = \big|\{i \in [n]\,|\,h(x_i) \neq y_i\}\big|/n = -q(\z,h)/n\,.$$ By Chernoff-Hoeffding bounds (see Theorem~\ref{thm:hoeff} in Appendix~\ref{sec:appchern}),
$$\Pr\big[|\trueerror(h)-\trueerror_T(h)| \geq \rho\big] \leq 2\exp(-2n\rho^2)$$ for all hypotheses $h\in\HHH_d$.
Hence,
\begin{eqnarray*}
\Pr\big[\mbox{$|\trueerror(h)-\trueerror_T(h)| \geq \rho$ for some $h \in \HHH_d$}\big]
\ifnum\full=0
\\
\fi
\leq 2 |\HHH_d| \exp(-2n\rho^2).
\end{eqnarray*}

We now analyze $\mathcal{A}^{\epsilon}_q (\mathrm{z})$ conditioned on the event that for all $h\in\HHH_d$, $|\trueerror(h)-\trueerror_T(h)| < \rho$.
For every $h \in \HHH_d,$ the probability that $\mathcal{A}^{\epsilon}_q (\mathrm{z})= h$ is
\begin{eqnarray*}
\ifnum\full=0
&&
\fi
\frac{\exp(-\frac{\epsilon}{2}\cdot n \cdot \trueerror_T(h))}{\sum_{h'\in\HHH_d} \exp(-\frac{\epsilon}{2}\cdot n \cdot \trueerror_T(h'))}
\ifnum\full=0
\\
\fi
& \leq &
\frac{\exp \left (-\frac{\epsilon}{2}\cdot n \cdot \trueerror_T(h) \right )}{\max_{h'\in\HHH_d} \exp(-\frac{\epsilon}{2}\cdot n \cdot \trueerror_T(h'))} \\
& = & \exp\left (-\frac{\epsilon}{2}\cdot n \cdot (\trueerror_T(h) - \min_{h'\in\HHH_d} \trueerror_T(h'))\right ) \\
& \leq & \exp\left (-\frac{\epsilon}{2}\cdot n \cdot (\trueerror_T(h) - (OPT + \rho))\right ).
\end{eqnarray*}
Hence, the probability that $\mathcal{A}^{\epsilon}_q (\mathrm{z})$ outputs a hypothesis $h\in\HHH_d $ such that $\trueerror_T(h) \geq OPT + 2\rho$ is at most $|\HHH_d|\exp(-\epsilon n \rho/2)$.

%To conclude the analysis,
Now set $\rho = \alpha/3$. If $\trueerror(h) \geq OPT + \alpha$ then $|\trueerror(h)-\trueerror_T(h)| \geq \alpha/3$ or $\trueerror_T(h) \geq OPT + 2\alpha/3$. Thus $\Pr[E] \leq |\HHH_d|(2\exp(-2n\alpha^2/9) + \exp(-\epsilon n \alpha/6)) \leq \beta$ where the last inequality holds for
%sofya: incorrect use of O-notation
$n %= O
\geq 6
\left((\ln | \HHH_d | + \ln \frac{1}{\beta}) \cdot \max\{\frac{1}{\epsilon\alpha},\frac{1}{\alpha^2}\}\right)$.
\end{proof}

\ifnum\full=1
\begin{remark}
In the non-private agnostic case, the standard Occam's razor bound  %(see, \eg~\cite{Kearns94})
 guarantees that $O((\log|\CCC_d|+\log (1/\beta))/\alpha^2)$ labeled examples suffice to agnostically learn a concept class $\CCC_d$. The bound of \thmref{PACvsPPAC} %shows that (Occam's razor) upper bounds on the sample size of private and non-private agnostic learners only differ
differs
 by a factor of $O(\frac{\alpha}{\eps})$ if $\alpha > \eps$, and does not differ at all
%if the same asymptotic complexity
 otherwise. For (non-agnostic) PAC learning, the dependence on $\alpha$ in the sample size for both the private and non-private versions improves to $1/\alpha$. In that case the upper bounds for private and non-private learners differ by  a factor of $O(1/\eps)$.   Finally, the theorem can be extended to settings where $\HHH_d \neq \CCC_d$, but in this case using the same sample complexity the learner outputs a hypothesis whose error is close to the best error attainable  by a function in $\HHH_d$.
\end{remark}

%\subsection{Implications of Private Agnostic Learner}\label{sec:implication}
\paragraph{Implications of the Private Agnostic Learner}

The private agnostic learner has the following important consequence: If some concept class $\CCC_d$ is learnable by any algorithm $\AAA$, not necessarily a private one, and $\AAA$'s output length in bits is polynomially bounded, then there is a (possibly exponential time) private algorithm that learns $\CCC_d$ using a polynomial number of samples. Since $\AAA$'s output is polynomially long, $\AAA$'s hypothesis class $\HHH_d$ must have size at most $2^{\poly(d)}$. Since $\AAA$ learns $\CCC_d$ using $\HHH_d$, class $\HHH_d$ must contain a good hypothesis. Thus, our private learner will learn $\CCC_d$ using $\HHH_d$  with  sample complexity linear in $\log|\HHH_d|$.

\paragraph{The ``compression version'' of Occam's razor} \anote{This text added to address referee's complaint.}

It is most natural to state our result as an analogue of the
cardinality version of Occam's razor, which bounds generalization
error in terms of the size of the hypothesis class. However, our
result can be extended to the compression version, which captures
the general relationship between compression and learning (we borrow the  ``cardinality version'' terminology from~\cite{Kearns94}). This latter
version states that any algorithm which ``compresses'' the data set,
in the sense that it finds a consistent hypothesis which has a short
description relative to the number of samples seen so far, is a good
learner (see \cite{BEHW87} and \cite[p.~34]{Kearns94}).

Compression by itself does not imply privacy, because the
compression algorithm's output might encode a few examples in the
clear (for example, the hyperplane output by a support vector machine is
defined via a small number of actual data points).
However, Theorem~\ref{thm:PACvsPPAC} can be extended to provide a
private analogue of the compression version of Occam's razor. If there
exists an algorithm that compresses, in the sense above, then there
also exists a private PAC learner which does not have fixed sample
complexity, but uses an expected number of samples similar to that of
the compression algorithm. The private learner proceeds in rounds: at
each round it requests twice as many examples as in the previous
round, and uses a restricted hypothesis class consisting of
sufficiently concise hypotheses from the original class $\HHH$. We
omit the straightforward details.

\subsection{Private Learning with VC dimension Sample Bounds} \label{sec:vclearn}

\anote{Rewrote this section completely based on referee's observation...}
In the non-private case one can also bound the sample size of a PAC
learner in terms of the Vapnik-Chervonenkis (VC) dimension of the concept class. % Let us recall the
% definition of the VC dimension (see, e.g., \cite{Kearns94}).
\begin{definition}[VC dimension] A set $S\subseteq X_d$ is \emph{shattered}
by a concept class $\CCC_d$ if $\CCC_d$ restricted to $S$ contains all $2^{|S|}$
possible functions from $S$ to $\{0,1\}$. The \emph{VC dimension} of $\CCC_d$,
denoted $VCDIM(\CCC_d)$, is the cardinality
of a largest set $S$ shattered by $\CCC_d$.
\end{definition}

We can extend \thmref{PACvsPPAC} to classes with finite VC dimension,
but the resulting sample complexity also depends logarithmically on
the size of the domain from which examples are drawn. Recent results
of Beimel \etal~\cite{BKN10} show that for ``proper'' learning, the
dependency is in fact necessary; that is, the VC dimension alone is not sufficient to bound the sample complexity of proper private
learning. It is unclear if the dependency is necessary in
general. \anote{Shiva/Kobbi: please confirm correctness.} \snote{Yes, it is correct.}

\begin{corollary} \label{thm:vcoccam}
%For any $d \in \N$, the class of concepts
Every concept class $\CCC_d$ is privately agnostically learnable using hypothesis class $\HHH_d=\CCC_d$ with
$n=O((VCDIM(\CCC_d)\cdot \ln |X_d  | +\ln \frac 1 {\beta})\cdot \max\{\frac{1}{\epsilon\alpha},\frac{1}{\alpha^2}\})$ labeled examples from $\DDD$. Here, $\eps, \alpha$, and $\beta$ are parameters of the private agnostic learner,
 and $VCDIM(\CCC_d)$ is the VC dimension of $\CCC_d$.
(The learner is not necessarily efficient.)
\end{corollary}

\begin{proof}
  Sauer's lemma~(see, \emph{e.g.}, \cite{Kearns94}) implies that there
  are  $O(|X_d|^{VCDIM(\CCC_d)})$ different labelings of
  $X_d$ by functions in $\CCC_d$. We can thus run the generic
  learner of the previous section with a hypothesis class of size
  $|\HHH_d| = O(|X_d|^{VCDIM(\CCC_d)})$. The statement follows directly.
\end{proof}

Our original proof of the corollary used a result of Blum,
Ligget and Roth~\cite{BLR08} (which was inspired, in turn, by our generic
learning algorithm) on generating synthetic data. The simpler proof
above was pointed out to us by an anonymous reviewer.

\paragraph{Remark: Computability Issues with Generic Learners}

In their full generality, the generic learning results of the previous sections (Theorems \ref{thm:PACvsPPAC} and \ref{thm:vcoccam})  produce well-defined randomized maps, but not necessarily ``algorithms'' in the sense of ``functions uniformly computable  by Turing machines''. This is because the concept class and example domain may themselves not be computable (nor even recognizable) uniformly (imagine, for example, a concept class indexed by elements of the halting problem). It is commonly assumed in the learning literature that elements of the concept class and domain can be computed/recognized by a Turing machine and some bound on the length of their binary representations is known. In this case, the generic learners can be implemented by randomized Turing machines with finite expected running time.

\section{An Efficient Private Learner for $\parity$}\label{sec:parity}
Let $\parity$ be the class of parity functions $c_{r}:\zo^{d}\rightarrow \zo$ indexed by $r\in\zo^d$, where $c_{r}(x)= r \odot x $ denotes the inner product modulo $2$.
In this section, we present an efficient private PAC learning algorithm for $\parity$. The main result is stated in Theorem~\ref{thm:parity-in-PPAC}.

The standard (non-private) PAC learner for $\parity$ \cite{helm,FischerSimon-1990} looks for the hidden vector $r$ by solving a system of linear equations imposed by examples $(x_i,c_r(x_i))$ that the algorithm sees. It outputs an arbitrary vector consistent with the examples, i.e., in the solution space of the system of linear equations. We want to design a private algorithm that emulates this behavior. A major difficulty is that the private learner's behavior must be  specified on {\em all} databases $\mathrm{z}$, even those which are not consistent with any single parity function. The standard PAC learner would simply fail in such a situation (we denote failure by the output $\perp$).  In contrast,  the probability that a private algorithm fails must be similar for all neighbors $\mathrm{z}$ and $\mathrm{z'}$.
% One important point is that the private algorithm has to be defined on all databases $\mathrm{z}$, even the ones which are not consistent with any parity function because the privacy guarantee has to hold for all neighboring databases $\mathrm{z}$ and $\mathrm{z'}$. In particular, we have to guarantee that the probability that the algorithm fails (does not find a consistent hypothesis and, in our notation, outputs $\perp$) is similar for all neighbors $\mathrm{z}$ and $\mathrm{z'}$. Observe that this is not true for the standard learning algorithm for $\parity$.

We first present a private algorithm $\AAA$ for learning $\parity$ that succeeds only with constant probability. %with failure probability $3/4$.
Later we amplify its success probability and get a private PAC learner~$\AAA^*$ for $\parity$. Intuitively, the reason \parity\ can be learned privately is that when a new example (corresponding to a new linear constraint) is added, the space of consistent hypotheses shrinks by at most a factor of 2. This holds unless the new constraint is inconsistent with previous constraints. In the latter case, the size of the space of consistent hypotheses goes to 0. Thus, the solution space changes drastically on neighboring inputs only when the algorithm fails (outputs $\perp$).
The fact that algorithm outputs $\perp$ on a database $\mathrm{z}$ and a valid (non $\perp$) hypothesis on a neighboring database $\mathrm{z'}$ might lead to privacy violations. To avoid this, our algorithm always outputs $\perp$ with probability at least $ 1/2$
on any input (Step~1).
\alg{A private learner for $\parity$, $\AAA(\mathrm{z},\eps)$}{
\begin{enumerate}
\item With probability $1/2$, output $\perp$ and terminate.

\item Construct a set
$S$ by picking each element of $[n]$ independently with probability $p=\eps/4$.

\item Use Gaussian elimination to solve the system of equations imposed by examples, indexed by $S$: namely, $\{x_i \odot r = c_r(x_i)\,: \, i\in S\}$. Let $V_{S}$ denote the resulting affine subspace.

\item Pick $r^*\in V_{S}$ uniformly
at random and output $c_{r^*}$; if $V_S = \emptyset$, output $\perp$.
\end{enumerate}
}

\ifnum\full=0
The (omitted)
proof of privacy is based on showing that the inclusion of any single point in the sample set $S$ increases the probability of a hypothesis being output by at most 2. The
(omitted)
proof of utility follows by considering all the possible situations in which the algorithm fails to satisfy the error bound, and by bounding the probabilities with which these situations occur.
\fi

\ifnum\full=0
\begin{lemma}[Privacy, Utility of $\AAA$]\label{lem:parity-utility} {\em (a)} Algorithm $\AAA$ is $\eps$-differentially private. {\em (b)} Let $\XXX$ be a distribution over $X=\{0,1\}^d$. Let $\mathrm{z}=(z_1,\dots,z_n)$, where every $z_i=(x_i,c(x_i))$ with $x_i$ drawn i.i.d.\ from  $\XXX$ and $c \in \parity$. If $n\geq \frac 8 {\eps \alpha}(d\ln 2+ \ln 4)$ then  $\Pr[\AAA(\mathrm{z},\eps)=h \text{ with } \trueerror(h)\leq \alpha]\geq \frac 1 4$.
\end{lemma}
\else
The
proof of $\AAA$'s utility follows by considering all the possible situations in which the algorithm fails to satisfy the error bound, and by bounding the probabilities with which these situations occur.

\begin{lemma}[Utility of $\AAA$]\label{lem:parity-utility}  Let $\XXX$ be a distribution over $X=\{0,1\}^d$. Let $\mathrm{z}=(z_1,\dots,z_n)$, where for all $i \in [n]$, the entry $z_i=(x_i,c(x_i))$ with $x_i$ drawn i.i.d.\ from  $\XXX$  and $c \in \parity$.  If $n\geq \frac 8 {\eps \alpha}\left(d\ln 2+ \ln 4 \right)$ then $$\Pr[\AAA(\mathrm{z},\eps)=h \text{ with } \mathrm{ \it error}(h)\leq \alpha]\geq \frac 1 4\:.$$
\end{lemma}
\begin{proof}
By standard arguments in learning theory \cite{Kearns94}, $\displaystyle |S| \geq \frac{1}{\alpha} \left(d\ln 2 +  \ln\frac{1}{\beta}\right)$ labeled examples are sufficient for learning \parity\ with error $\alpha$ and failure probability $\beta$.
Since $\AAA$ adds each element of $[n]$ to $S$ independently with probability $p=\eps/4$,  the expected size of $S$ is  $pn=\eps n/4$. By the Chernoff bound  (Theorem~\ref{thm:chern}), $|S|\geq \eps n/8$ with probability at least $1-e^{-\eps n/16}$.
We set $\beta=\frac 1 4$ and pick $n$ such that $\eps n/8 \geq \frac{1}{\alpha} \left(d\ln 2 + \ln 4\right)$.

We now bound the overall success probability.
$\AAA(\mathrm{z},\eps)=h \text{ with } \trueerror(h)\leq \alpha$ unless one of the following bad events happens: (i) $\AAA$ terminates in Step 1, (ii) $\AAA$ proceeds to Step 2, but does not get enough examples: $|S|< \frac{1}{\alpha} \left(d\ln 2 +  \ln 4)\right)$, (iii) $\AAA$ gets enough examples, but outputs a hypothesis with error greater than $\alpha$. The first bad event occurs with probability 1/2. If the lower bound on the database size $n$ is satisfied then %, by the Chernoff bound  (Theorem~\ref{thm:chern}),
the second bad event occurs with probability at most $e^{-\eps n/16}/2\leq 1/8$. The last inequality follows from the bound on $n$ and the fact that $\alpha\leq 1/2$. Finally, by our choice of parameters, the last bad event occurs with probability at most $\beta/2=1/8$. The claimed bound on the success probability follows.
\end{proof}

\begin{lemma} [Privacy of $\AAA$]  \label{lem:parity-privacy}Algorithm $\AAA$ is $\eps$-differentially private.
\end{lemma}
As mentioned above, the key observation in the following proof is that including of any single point in the sample set $S$ increases the probability of a hypothesis being output by at most 2.
\begin{proof}
To show that $\AAA$ is $\eps$-differentially private, it suffices to prove that any output of $\AAA$, either a valid hypothesis or $\perp$, appears with roughly the same probability on neighboring databases $\mathrm{z}$ and $\mathrm{z'}$. In the remainder of the proof we fix $\epsilon$, and write $\AAA(\z)$ as shorthand for $\AAA(\z,\eps)$. We have to show that
\begin{eqnarray}
\label{eq:prob-h-bound}{\Pr[\AAA(\mathrm{z})=h]}\leq e^\eps \cdot {\Pr[\AAA(\mathrm{z'})=h]} & \text{for all neighbors $\mathrm{z},\mathrm{z'}\in D^n$} \text{ and all hypotheses $h\in \parity$;}\\
\label{eq:prob-perp-bound}{\Pr[\AAA(\mathrm{z})=\perp]}\leq e^\eps \cdot {\Pr[\AAA(\mathrm{z'})=\perp]}&\text{for all neighbors $\mathrm{z},\mathrm{z'}\in D^n$.}
\end{eqnarray}
We prove the correctness of Eqn.~\eqref{eq:prob-h-bound} first.  Let $\mathrm{z}$ and $\mathrm{z'}$ be neighboring databases, and let $i$ denote the entry on which they differ. Recall that $\AAA$ adds $i$ to $S$ with probability $p$. Since $\mathrm{z}$ and $\mathrm{z'}$ differ only in the $i^{th}$ entry, $\Pr[\AAA(\mathrm{z})=h \ | \  i\notin S]=\Pr[\AAA(\mathrm{z'})=h \ | \  i\notin S]$.

%To see that Eqn.~\eqref{eq:prob-h-bound} holds note first
Note that if $\Pr[\AAA(\mathrm{z'})=h \ | \  i\notin S] =0$, then also $\Pr[\AAA(\mathrm{z})=h \ | \  i\notin S] =0$, and hence $\Pr[\AAA(\mathrm{z})=h] =0$ because adding a constraint does not add new vectors to the space of solutions.
Otherwise, $\Pr[\AAA(\mathrm{z'})=h \ | \  i\notin S] > 0$. In this case, we rewrite the probability on $\mathrm{z}$ as follows:
$$\Pr[\AAA(\mathrm{z})=h]= p \cdot \Pr[\AAA(\mathrm{z})=h \ | \  i\in S] + (1-p) \cdot \Pr[\AAA(\mathrm{z})=h \ | \  i\notin S],$$
and apply the same transformation to the probability on $\mathrm{z'}$.
Then
\begin{eqnarray}
\frac{\Pr[\AAA(\mathrm{z})=h]}{\Pr[\AAA(\mathrm{z'})=h]}
&=& \frac{p \cdot \Pr[\AAA(\mathrm{z})=h \ | \  i\in S] + (1-p) \cdot \Pr[\AAA(\mathrm{z})=h \ | \  i\notin S]}{p\cdot \Pr[\AAA(\mathrm{z'})=h \ | \  i\in S] + (1-p) \cdot \Pr[\AAA(\mathrm{z'})=h \ | \  i\notin S]} \nonumber \\
&\leq& \frac{p \cdot \Pr[\AAA(\mathrm{z})=h \ | \  i\in S] + (1-p) \cdot \Pr[\AAA(\mathrm{z})=h \ | \  i\notin S]}{p\cdot 0 + (1-p) \cdot \Pr[\AAA(\mathrm{z'})=h \ | \  i\notin S]} \nonumber \\
&=&\frac p {1-p}\cdot\frac{\Pr[\AAA(\mathrm{z})=h \ | \  i\in S]}{\Pr[\AAA(\mathrm{z})=h \ | \  i\notin S]}+1 \label{eqn:valid-hyp}
% \leq\frac{2p}{1-p} +1\leq \eps+1\leq e^\eps.
\end{eqnarray}

We need the following claim:

\begin{claim} \label{claim:ratio-bound}
$\displaystyle\frac{\Pr[\AAA(\mathrm{z})=h \ | \  i\in S]}{\Pr[\AAA(\mathrm{z})=h \ | \  i\notin S]}\leq 2$, for all $\mathrm{z}\in D^n$ and all hypotheses $h\in \parity$.
\end{claim}

This claim is proved below. For now, we can plug it into Eqn.~\eqref{eqn:valid-hyp} to get

$$ \frac{\Pr[\AAA(\mathrm{z})=h]}{\Pr[\AAA(\mathrm{z'})=h]}
\leq\frac{2p}{1-p} +1\leq \eps+1\leq e^\eps\,.$$

The first inequality holds since $p=\eps/4$ and $\eps\leq 1/2$. This establishes Eqn.~\eqref{eq:prob-h-bound}.
%In the last line, the first inequality follows from Claim~\ref{claim:ratio-bound}, proved below, and the second inequality holds since $p=\eps/4$ and $\eps\leq 1/2$.
The proof of  Eqn.~\eqref{eq:prob-perp-bound} is similar:
\begin{eqnarray*}
\frac{\Pr[\AAA(\mathrm{z})=\perp]}{\Pr[\AAA(\mathrm{z'})=\perp]}
&=& \frac{p \cdot \Pr[\AAA(\mathrm{z})=\perp \ | \  i\in S] + (1-p) \cdot \Pr[\AAA(\mathrm{z})=\perp \ | \  i\notin S]}{p\cdot \Pr[\AAA(\mathrm{z'})=\perp \ | \  i\in S] + (1-p) \cdot \Pr[\AAA(\mathrm{z'})=\perp \ | \  i\notin S]} \\
&\leq& \frac{p\cdot 1  + (1-p) \cdot \Pr[\AAA(\mathrm{z})=\perp \ | \  i\notin S]}{p\cdot 0 + (1-p) \cdot \Pr[\AAA(\mathrm{z'})=\perp \ | \  i\notin S]}\\
&=&\frac{p}{(1-p) \cdot \Pr[\AAA(\mathrm{z'})=\perp \ | \  i\notin S]}+1
\leq\frac{2p}{1-p} +1\leq \eps+1\leq e^\eps.
\end{eqnarray*}
In the last line, the first inequality follows from the fact that on any input, $\AAA$ outputs $\perp$ with probability at least $1/2$. This completes the proof of the lemma.
\end{proof}

We now prove Claim~\ref{claim:ratio-bound}.
% The following claim was used in the proof of \lemref{parity-privacy}.
% \begin{claim} \label{claim:ratio-bound}
% $\displaystyle\frac{\Pr[\AAA(\mathrm{z})=h \ | \  i\in S]}{\Pr[\AAA(\mathrm{z})=h \ | \  i\notin S]}\leq 2$, for all $\mathrm{z}\in D^n$ and all hypotheses $h\in \parity$.
% \end{claim}
\begin{proof}[Proof of Claim~\ref{claim:ratio-bound}]
The left hand side
\begin{eqnarray*}
\frac{\Pr[\AAA(\mathrm{z})=h \ | \  i\in S]}{\Pr[\AAA(\mathrm{z})=h \ | \  i\notin S]}
&=&
\frac{\sum_{T \subseteq [n] \setminus \{i\}}\Pr[\AAA(\mathrm{z})=h \ | \ S =T\cup \{i\}]\cdot \Pr[\text{$\AAA$ selects $T$ from $[n]\setminus \{i\}$}]}{\sum_{T \subseteq [n] \setminus \{i\}}\Pr[\AAA(\mathrm{z})=h \ | \ S =T]\cdot \Pr[\text{$\AAA$ selects $T$ from $[n]\setminus \{i\}$}]}.
\end{eqnarray*}
To prove the claim, it is enough to show that
$\displaystyle\frac{ \Pr[\AAA(\mathrm{z})=h \ | \ S =T\cup \{i\}]}{\Pr[\AAA(\mathrm{z})=h \ | \ S =T]}\leq 2$
for each $T \subseteq [n] \setminus \{i\}$. Recall that $V_S$ is the space of solutions to the system of linear equations $\{\langle x_i, r \rangle = c_r(x_i)\,: \, i\in S\}$. Recall also that $\AAA$ picks $r^*\in V_S$ uniformly at random and outputs $h=c_{r^*}$. Therefore,
$$\Pr[\AAA(\mathrm{z})=c_{r^*} \ | \ S]= \left \{ \begin{array}{ll}
1/|V_S| & \mbox{ if } r^* \in V_S, \\
0 & \mbox{ otherwise. }
\end{array}
\right. $$
If $\Pr[\AAA(\mathrm{z})=h \ | \ S =T]=0$ then $\Pr[\AAA(\mathrm{z})=h \ | \ S =T\cup \{i\}]=0$ because a new constraint does not add new vectors to the space of solutions. If $\Pr[\AAA(\mathrm{z})=h \ | \ S =T\cup\{i\}]=0$, the required inequality holds.
If neither of the two probabilities is 0,
$$\frac{\Pr[\AAA(\mathrm{z})=h \ | \ S =T\cup\{i\}]}{\Pr[\AAA(\mathrm{z})=h \ | \ S =T]}
=\frac{1/|V_{T\cup\{i\}}|}{1/|V_T|}=\frac{|V_T|}{|V_{T\cup\{i\}}|}\leq 2.$$
The last inequality holds because in $\Z_2$ (the finite field with 2 elements where arithmetic is performed modulo 2),
%sofya: need to define everything; this notation (Z_p) is used in math for the ring of p-adic integers.
 adding a consistent linear constraint either reduces the space of solutions by a factor of 2 (if the constraint is linearly independent from $V_T$) or does not change the solutions space (if it is linearly dependent on the previous constraints). The constraint indexed by $i$ has to be consistent with constraints indexed by $T$, since both probabilities are not~$0$.
\end{proof}
\fi

\newcommand{\perterror}{{\widehat{\text{\it err}}_T}}
\newcommand{\trainerror}{{\text{\it err}_T}}

It remains to amplify the success probability of $\AAA$. To do so, we construct a private version of the standard
(non-private) algorithm for amplifying a learner's success probability. The standard amplification algorithm generates a set of
hypotheses by invoking $\AAA$ multiple times on independent examples, and then outputs a hypothesis
from the set with the least training error as evaluated on a fresh test set (see~\cite{Kearns94} for details). Our
private amplification algorithm differs from the standard algorithm only in the last step: it adds Laplacian noise to the training
error to obtain a private version of the error, and then uses the perturbed training error instead of
the true training error to select the best hypothesis from the set.~\footnote{Alternatively, we could use the generic learner from \thmref{PACvsPPAC} to select among the candidate hypotheses; the resulting algorithm has the same asymptotic behavior as the algorithm we discuss here. We chose the algorithm that we felt was simplest.}\anote{Added footnote to address referee's comment.}
%Let $\mathrm{z}=(z_1,\dots,z_n)$,  where  every $z_i=(x_i,c(x_i))$ with $x_i$ drawn i.i.d.\ from $\mathcal{X}$ and $c \in \parity$.
%In the following discussion, we will assume that $n \geq Cd \log(1/\beta)/(\eps\alpha)$, for a constant $C$ to be specified.
%We split $\mathrm{z}$ of size $n$ into two parts, $\bar{\mathrm{z}}$ and $\hat{\mathrm{z}}$. We will use the $\hat{\mathrm{z}}$ as the test set.
Recall that $\Lap(\lambda)$ denotes the Laplace probability distribution with mean $0$, standard deviation $\sqrt{2}\lambda$, and p.d.f.\ $f(x)=\frac{1}{2\lambda}e^{-|x|/\lambda}$.
%Let $k = \ln (3/\beta), n' = cd/(\eps\alpha)$ (for some constant $c$), $s = c'\max\{1/\alpha,\alpha\eps\}\ln(3k /\beta)$, (for some constant $c'$), and let $n=kn'+s$.  We split $\mathrm{z}$ into two parts. Let $\bar{\mathrm{z}}=(z_1,\dots,z_{kn'})$ We divide $\bar{\mathrm{z}}$ into $k$ equal parts each of size $n'$. Let $\bar{\mathrm{z}}_j=(z_{(j-1)n'+1},\dots,z_{jn'})$ for $j \in [k]$. Let $\hat{\mathrm{z}}=(z_{kn'+1},\dots,z_{kn'+s})$.

\alg{Amplified private PAC learner for \parity,  $\AAA^*(\mathrm{z},\eps,\alpha,\beta)$}{
\begin{enumerate}
\item $\beta' \gets \frac \beta 2$; $\alpha'\gets \frac \alpha 5$;
$k \gets \left\lceil\log_{\frac 3 4} \left(\frac 1 {\beta'}\right)\right\rceil$; $n' \gets \frac {cd}{\eps\alpha'}$; $s \gets \frac{c'k}{\alpha'\eps}\log\left(\frac{k} {\beta'}\right)$
(where $c,c'$  are  constants).
\item If $n\leq kn'+s$, stop and return ``insufficient samples''.
%\anote{have to handle $n>kn'+s$.}
\item Divide $\z=(z_1,\dots,z_n)$ into two parts, training set $\bar{\mathrm{z}}=(z_1,\dots,z_{kn'})$ and test set $\hat{\mathrm{z}}=(z_{kn'+1},\dots,z_{kn'+s}).$
\item Divide $\bar{\mathrm{z}}$ into $k$ equal parts each of size $n'$, let $\bar{\mathrm{z}}_j=(z_{(j-1)n'+1},\dots,z_{jn'})$ for $j \in [k].$
\item\label{step:compute-hypotheses-and-perror} For $j \gets 1 \text{ to } k$ \\
\hspace*{7mm}  $\displaystyle h_j \gets\AAA(\bar{\mathrm{z}}_j,\eps)$; \\
\hspace*{7mm}  set perturbed training error of $h_j$ to $\displaystyle \perterror (h_j) = \frac {\big|\{z_i \in \hat{\mathrm{z}} \, :\, h_j(x_i) \neq c(x_i)\}\big|}{s}   +
\Lap\left(\frac k {s\eps}\right)$.
\item Output $h^*=h_{j^*}$ where %$\perterror(h^*)
$j^*= \operatorname{argmin}_{j \in [k]} \{\perterror(h_j)\}$.
\end{enumerate}
}

\begin{theorem}\label{thm:parity-in-PPAC}
Algorithm $\AAA^*$ efficiently and privately PAC learns \parity\ (according to Definition~\ref{def:private-general}) with $O\left(\frac{d \log(1/\beta)}{\eps\alpha}%+\frac{\log^2(1/\beta)}{\eps\alpha}
\right)$ samples.
\end{theorem}
The theorem follows from Lemmas~\ref{lem:parity-amplified-privacy} and~\ref{lem:parity-amplified-utility} that, respectively, prove privacy and utility of $\AAA^*$.
\begin{lemma}[Privacy of $\AAA^*$]\label{lem:parity-amplified-privacy}
Algorithm $\AAA^*$ is $\eps$-differentially private.
\end{lemma}
\begin{proof}
We prove that even if $\AAA^*$ released all hypotheses $h_j$, computed in Step~\ref{step:compute-hypotheses-and-perror}, together
with the corresponding perturbed error estimates $\perterror(h_j)$, it would still be $\eps$-differentially private. Since the output of $\AAA^*$
can be computed solely from this information, Claim~\ref{claim:composition} implies that $\AAA^*$ is $\eps$-differentially private.

By Lemma~\ref{lem:parity-privacy}, algorithm $\AAA$ is $\epsilon$-differentially private. Since $\AAA$ is invoked on disjoint parts of $\mathrm{z}$ to
compute hypotheses $h_j$,
releasing all these hypotheses would also be $\eps$-differentially private.

Define the training error of hypothesis $h_j$ on $\hat{\mathrm{z}}$ as $\trainerror(h_j) = |\{z_i \in \hat{\mathrm{z}} \, :\, h_j(x_i) \neq c(x_i)\}|/s$.
The global sensitivity of the ${\trainerror}$ function is $1/s$ because
$|\trainerror(\mathrm{z})-\trainerror(\mathrm{z'})| \leq 1/s$ for every
pair of neighboring databases $\mathrm{z},\mathrm{z'}$.  Therefore, by Theorem~\ref{thm:DMNS}, releasing $\perterror(h_j)$
for one $j$, would be $\eps/k$-differentially private, and by Claim~\ref{claim:composition}, releasing all $k$ of them
would be $\eps$-differentially private. Since hypotheses $h_j$ and their perturbed errors $\perterror(h_j)$ are
computed on disjoint parts of the database $\z$, releasing all that information would still be $\eps$-differentially
private.
\end{proof}
\begin{lemma}[Utility of $\AAA^*$]\label{lem:parity-amplified-utility}
$\AAA^*(\cdot,\eps,\cdot,\cdot)$  PAC learns $\parity$ with sample complexity $n = O(\frac {d \log(1/\beta)}{\eps\alpha})$.
\end{lemma}

\ifnum\full=1
\begin{proof}
Let $\XXX$ be a distribution over $X=\{0,1\}^d$. Recall that $\z=(z_1,\dots,z_n)$, where for all $i \in [n]$, the entry $z_i=(x_i,c(x_i))$ with $x_i$ drawn i.i.d.\ from $\XXX$ and $c \in \parity$.
Assume that $\beta<1/4$, and $n \geq C\frac {d \log(1/\beta)}{\eps\alpha}$ for a constant $C$ to be determined.
We wish to prove that
 $\Pr[\trueerror(h^*) \leq \alpha] \geq 1-\beta$, where $h^*$ is the hypothesis output by $\AAA^*$.

Consider the set of candidate hypotheses $\{h_1,...,h_k\}$ output by the invocations of $\AAA$ inside of $\AAA^*$.  We call a hypothesis $h$ {\em good} if $\trueerror({h}) \leq \frac \alpha 5 = \alpha'$. We call a hypothesis ${h}$ {\em bad} if $\trueerror({h}) \geq \alpha=5\alpha'$. Note that good and bad refer to a hypothesis' true error rate on the underlying distribution.

We will show:
\begin{enumerate}
\item With probability at least $1-\beta'$, one of the invocations of $\AAA$ outputs a {\em good} hypothesis. %The probability here is taken over the coins of the nvocations of $\AAA$, and over the choice of the first $kn'$ data points (that is, $\bar \z$).

\item Conditioned on any particular outcome  $\{h_1,...,h_k\}$ of the invocations of $\AAA$, with probability at least $1-\beta'$,  both:
  \begin{enumerate}
  \item Every \emph{good} hypothesis  $h_j$ in $\{h_1,...,h_k\}$ has training error $\trainerror(h_j)\leq
    2\alpha'$.
  \item Every  {\em bad} hypothesis $h_j$
    in $\{h_1,...,h_k\}$ has training error $\trainerror(h_j)\geq
    4\alpha'$.
  \end{enumerate}
\item Conditioned on any particular hypotheses  $\{h_1,...,h_k\}$ and training errors $\trainerror(h_1),...,\trainerror(h_k)$,  with probability at least $1-\beta'$, for all $j$ simultaneously, $|\perterror(h_j)-\trainerror(h_j)|< \alpha'$.
\end{enumerate}

Suppose the events described in the three claims above all occur. Then some good hypothesis has perturbed training error
less than $3\alpha'$, yet all bad hypotheses have perturbed training error greater than $3\alpha'$. Thus, the hypothesis $h_{j^*}$ with minimal perturbed error $\perterror(h_{j^*})$ is {\em not bad}, that is, has true error at most $\alpha$. By the claims above, the probability that all three events occur is at least $1-3\beta' =1-\beta$, and so the lemma holds. We now prove the claims.

First, by the utility guarantee of $\AAA$, each invocation of $\AAA$ inside $\AAA^*$ outputs a {\em good} hypothesis with probability at least $\frac 1 4$ as long as the constant $c> 8(\ln 2+\ln 4)$ (since in that case $n'$, the size of each $\bar{\mathrm{z}}_j$, is large enough to apply \lemref{parity-utility}). The  $k$ invocations of the algorithm $\AAA$ are
 on independent samples, so the probability that none of $h_1,\dots,h_k$ is good  is
 at most $\left(\frac 3 4\right)^k$. Setting $k\geq \log_{\frac 3 4} \frac 1{\beta'}$  ensures that with probability at least $1-\beta'$, at least one of $h_1,\dots,h_k$ has error at most $\alpha'$.

Second, fix a particular sequence of candidate hypotheses $h_1,...,h_k$. For each $j$, the training error $\trainerror(h_j)$ is the average of $s$ Bernouilli trials, each with success probability $\trueerror(h_j)$. (Crucially, the training set $\hat z$ is independent of the data $\bar z$ used to find the candidate hypotheses). To bound the training error, we apply the multiplicative Chernoff bound (Theorem~\ref{thm:chern}) with $n=s$ and $p=\trueerror(h_j)$. Here, $p\leq \alpha'$ if $h_j$ is \emph{good}, and $p\geq 5\alpha'$ if $h_j$ is \emph{bad}.

By the multiplicative Chernoff bound (Theorem~\ref{thm:chern}) if $s \geq \frac {c_1}{\alpha'} \ln \frac k{\beta'}$ (for appropriate constant $c_1$), then
\begin{align*}
  \Pr\left[\trainerror(h_j) \geq 2\alpha' \, \big| \, h_j \text{ is
      good}\right] & \leq \Pr[\text{Binomial}(s,\alpha') \geq 2\alpha's]\leq
  \frac {\beta'} k\, , \ \text{and}
  \\
  \Pr\left[\trainerror(h_j) \leq 4\alpha' \, \big| \, h_j \text{ is
      bad}\right]  & \leq \Pr[\text{Binomial}(s,5\alpha') \leq 4\alpha's] \leq
  \frac {\beta'} k.
\end{align*}
By a union bound, all the training errors are (simultaneously) approximately correct, with probability at least $1-k\cdot\frac{\beta'}{k} = 1-\beta'$.

Finally, we prove the third claim. Consider a particular candidate hypothesis $h_j$. If $s \geq \frac{c_2k}{\alpha'\eps} \ln \frac k{\beta'}$ (for appropriate constant $c_2$),
 then (by using the c.d.f.\footnote{The cumulative distribution function of the Laplacian distribution
 $\Lap(\lambda)$ is $F(x) = \frac 1 2 \exp\left(\frac x \lambda\right)$ if $x < 0$ and $1-\frac 1 2 \exp\left(-\frac x \lambda\right)$ if $x \geq 0$.}
 of the Laplacian distribution)
 $$\Pr\left[\left|\trainerror({h_j}) -\perterror({h_j})\right| < \alpha'\right] = \Pr\left[\text{Lap}\left(\frac k {s\eps}\right)\geq \alpha'\right]\leq
 \frac{\beta'}k.$$
By a union bound, all $k$ perturbed estimates are within $\alpha'$ of their correct value with probability at least $1-k\cdot\frac{\beta'}{k}=1-\beta'$. This probability is taken over the choice of Laplacian noise, and so the bound holds independently of the particular hypotheses or their training error estimates.
\end{proof}

\begin{remark} In the non-private case $O((d+\ln (1/\beta))/\alpha)$ labels are sufficient for
learning $\parity$. Theorem~\ref{thm:parity-in-PPAC} shows that the upper bounds on the sample size of private
and non-private learners differ only by a factor of $O(\ln(1/\beta)/\eps)$.\end{remark}

\section{Local Protocols and SQ learning}
\label{sec:sqlocal}
In this section, we relate private learning in the local model to  the SQ model of Kearns
~\cite{Kearns98}. We first define the two models precisely. We then prove their equivalence (\secref{equivalence}), and discuss the implications for learning (\secref{locallearn}). Finally, we define the concept class $\mparity$ and prove that it separates interactive from noninteractive local learning (\secref{mparity}).

\paragraph{Local Model.} We start by describing private computation in the local model. Informally, each individual holds her private information locally, and hands it to the learner after randomizing it. This is modeled %by considering a database $\mathrm{z}=(z_1,\dots,z_n) \in D^n$ containing $n$ entries from some domain $D$, and
by letting the local algorithm access each entry $z_i$ in the input database $\mathrm{z}=(z_1,\dots,z_n) \in D^n$ only via {\em local randomizers}.

\begin{definition}[Local Randomizer]
An {\em $\eps$-local randomizer} $R: D\rightarrow W$ is an $\eps$-differentially private algorithm
%sofya: added
that takes a database of size $n=1$.
That is, $\Pr[R(u)=w] \leq e^\eps \Pr[R(u')=w]$ for all $u,u' \in D$ and all $w \in W$. The probability is taken over the coins of $R$ (but {\em not} over the choice of the input).
\end{definition}
Note that since a local randomizer works on a data set of size $1$, $u$ and $u'$ are neighbors for all  $u,u' \in D$. Thus, this definition is consistent with our previous definition of differential privacy.

\begin{definition}[LR Oracle] Let $\mathrm{z}=(z_1,\dots,z_n) \in D^n$ be a database.
 An {\em LR oracle} $LR_{\mathrm{z}}(\cdot,\cdot)$ gets an index $i\in[n]$ and an $\eps$-local randomizer $R$, and outputs a random value $w\in W$ chosen according to the distribution $R(z_i)$. The distribution $R(z_i)$ depends only on the entry $z_{i}$ in $\mathrm{z}$.
\end{definition}

%sofya: added $\eps$- and changed some $i$'s to $j$'s to disambiguate indices and runs
\begin{definition}[Local algorithm]
An algorithm is {\em $\eps$-local} if it accesses the database $\mathrm{z}$ via the oracle $LR_{\mathrm{z}}$ with the following restriction:
% sofya: not clear why we needed the following line
%for all runs of the local algorithm and
for all $i\in[n]$, if $LR_{\mathrm{z}}(i,R_1),\ldots,LR_{\mathrm{z}}(i,R_k)$ are the algorithm's invocations of $LR_{\mathrm{z}}$ on index $i$, where each $R_j$  is an $\eps_j$-local randomizer, then  $\eps_1+\cdots+\eps_k \leq \eps$.

Local algorithms that prepare all their queries to $LR_{\mathrm{z}}$ before receiving any answers are called
\emph{noninteractive}; otherwise, they are  \emph{interactive}.
\end{definition}
%Since $\eps_1+\cdots+\eps_k \leq \eps$,
By Claim~\ref{claim:composition}, $\eps$-local algorithms are $\eps$-differentially private.

\paragraph{SQ Model.} In the statistical query (SQ) model, algorithms access statistical properties of a distribution rather than individual examples.
\begin{definition}[SQ Oracle]
Let $\DDD$ be a distribution over a domain $D$. An \emph{SQ oracle} $SQ_{\DDD}$ takes as input a function $g: D \rightarrow\oo$ and
%sofya: it is clear from the condition on \tau that it is a real number
%a real number
a tolerance parameter
$\tau \in (0,1)$; it outputs $v$ such that: $$|v-\E_{u\sim \DDD}[g(u)]|\leq \tau.$$
\end{definition}

\ifnum\full=1
The query function $g$ does not have to be Boolean. Bshouty and Feldman \cite{BshoutyFeldman-2001} showed that given access to an SQ oracle which accepts only boolean query functions, one can simulate an oracle that accepts real-valued functions $g:D \rightarrow [-b,b]$, and outputs $\E_{u \sim \DDD}[g(u)]\pm \tau$ using $O(\log (b/\tau))$ nonadaptive queries to the SQ oracle and similar processing time.
\fi

\begin{definition}[SQ algorithm]
An {\em SQ algorithm} accesses the distribution $\DDD$ via the SQ oracle $SQ_\DDD$.
SQ algorithms that prepare all their queries to $SQ_\DDD$ before receiving any answers are called \emph{nonadaptive}; otherwise, they are called {\em adaptive}.
\end{definition}
\ifnum\full=1
Note that we do not restrict $g()$ to be efficiently computable. We will distinguish later those algorithms that only make queries to efficiently computable functions $g()$.
\fi

\subsection{Equivalence of Local and SQ Models}
\label{sec:equivalence}

Both the SQ and local models restrict algorithms to access inputs in a particular manner. There is a significant difference though: an SQ oracle sees a distribution $\DDD$, whereas a local algorithm takes as input a fixed (arbitrary) database $\mathrm{z}$. Nevertheless, we show that if the entries of $\mathrm{z}$ are chosen i.i.d.\ according to $\DDD$, then the models are equivalent.  Specifically, an algorithm in one model can \emph{simulate} an algorithm in the other model.  Moreover,
%sofya: replaced samples sizes with query complexity, since that's what our results relate
%the ``sample sizes'' are
%shiva: now not correct as we talking about expected query complexity
%ads 8/3/08: fixed with ``expected'' query complexity
the expected query complexity is preserved up to polynomial factors. \ifnum\full = 0 In the full version of this paper~\cite{KLNRS08}, we also investigate the efficiency of these simulations. \fi
%, where "sample size" means the number of queries in the SQ model, and the number of database entries in the local model.
\ifnum\full=1 %ads 8/2/08
We first present the simulation of SQ algorithms by local algorithms (Section~\ref{sec:local-sq}). The simulation in the other direction is more delicate and is presented in Section~\ref{sec:sq-local}.
\fi

%shiva-line not needed any more
%We now summarize the main consequences of these simulations. Roughly, a local learner (resp.\ SQ learner) is an algorithm that is a valid PAC learner but accesses its input (resp.\ input distribution) according to the limitations of the local (resp.\ SQ) model.

\ifnum\full=1
\subsubsection{Simulation of SQ Algorithms by Local Algorithms} \label{sec:local-sq}
\else
\subsubsection{Local Simulation of SQ Algorithms} \label{sec:local-sq}
\fi
%Sofya: changed a) to remove "learning" from the Blum etal result; b) to remove refs to "centralized" model
Blum \etal\cite{BDMN05} used the fact that sum queries can be answered privately with little noise to show that any efficient SQ algorithm can be simulated privately and efficiently.  We show that
%this result also holds in the local model,
it can be simulated efficiently even by a local algorithm, albeit with slightly worse parameters.

Let $g:D\rightarrow$
\ifnum\full=0
$\{+1,-1\}$
\else
$[-b,b]$
\fi
be the SQ query we want to simulate. By Theorem~\ref{thm:DMNS}, since the global sensitivity of $g$ is $2b$,
the algorithm $R_g(u)=g(u)+\eta$ where $\eta \sim \Lap(2b/\eps)$ is an $\eps$-local randomizer. We construct a local algorithm $\AAA_g$ that, given $n$ and $\eps$, and access to a database $\z$ via oracle $LR_{\mathrm{z}}$,
invokes $LR_\z$ for every $i\in[n]$  with the
randomizer $R_g$  and outputs the average of the responses:
\alg{A local algorithm $\AAA_g(n, \eps, LR_\z)$ that simulates an SQ query $g:D\rightarrow[-b,b]$}
{\begin{enumerate}
\item Output $\frac 1 n \sum_{i=1}^n LR_\z(i, R_g)$ where $R_g (u) = g(u) + \eta$ and $\eta \sim Lap\left(\frac{2b}{\eps}\right)$.
\end{enumerate}}

Note that $\AAA_g$ outputs $\left(\frac1 n\sum_{i=1}^n g(z_i)\right) + \left(\frac1 n\sum_{i=1}^n\eta_i\right)$, where the $\eta_i$ are i.i.d. from $\Lap\left(\frac{2b}\eps\right)$.
This algorithm is $\eps$-local (since it applies a single $\eps$-local randomized to each entry of $\z$), and therefore $\eps$-differentially private. The following lemma
shows that when the input database $\z$ is large enough, $\AAA_g$ simulates the desired SQ query $g$ with small error probability.

\begin{lemma}\label{lem:sim1}
 \ifnum\full=0
 If database $\z$ has $n=\Omega(\log(1/\beta)\eps^{-2}\tau^{-2})$
\else
If, for sufficiently large constant $c$, database $\z$ has $n\geq c \cdot \frac{\log(1/\beta)b^2}{\eps^{2}\tau^{2}}$
\fi
entries sampled i.i.d.\ from a distribution $\DDD$ on $D$ then algorithm
$\AAA_g$ approximates $\E_{u\sim\DDD}[g(u)]$ within additive error $\pm \tau$ with probability at least $1-\beta$.
\end{lemma}
\ifnum\full=0
\paragraph{Simulation.}  Consider an SQ algorithm making at most $t$ queries to $SQ_\DDD$. Our local algorithm simulates each query using \lemref{sim1} with parameters $\beta'=\beta/t$, $\tau$, and $\eps$, on a database $\mathrm{z}$ containing $O(t \log(1/\beta')\eps^{-2}\tau^{-2})$ entries sampled from $\DDD$. Each query is simulated with a fresh portion of $\mathrm{z}$, and hence privacy is preserved as each entry is subjected to a single application of the $\eps$-local randomizer $R$. By the union bound, the probability that any of the queries is not  approximated within additive error $\tau$ is at most~$\beta$.
\fi
\ifnum\full=1
\begin{proof}
Let $v=\E_{u \sim \DDD}[g(u)]$ denote the true mean. By the Chernoff-Hoeffding bound for real-valued variables  (Theorem~\ref{thm:hoeff}),
$$\textstyle\Pr \left [ \left |\frac{1}{n}\sum_{i=1}^n g(u_i) - v \right | \geq \frac\tau 2 \right ] \leq 2\exp\left(-\frac{\tau^2 n }{8b^2}\right).$$
Therefore, in the absence of additive Laplacian random noise, $O\left (\frac{\ln(1/\beta)b^2}{\tau^2}\right )$ examples are enough to approximate $\E_{u \sim \DDD}[g(u)]$ within additive error $\pm \frac \tau 2$ with probability at least $1-\frac\beta 2$.
(Note that the number of examples is smaller than the lower bound on $n$ in the lemma by a factor of $O(\eps^{-2})$).

%ads 3/5/09: moved lemma to appendix
%The following claim analyzes how the Laplacian noise affects the accuracy of the simulation. (The proof is standard, but we did not find the explicit statement we needed in the literature.)

The effect of the Laplace noise can also be bounded via a standard tail inequality: setting
%If the Laplace random variables $X_i$'s have parameter
$\lambda =\frac{2b}\eps$ in  \lemref{laplace}, we get that  $O\left (\frac{\ln(1/\beta)b^2}{\eps^{2}\tau^{2}}\right )$
samples are sufficient to ensure that the average of $\eta_i$'s lies outside $[-\frac\tau 2,\frac\tau 2]$ with
probability at most $\frac\beta 2$. It follows that $\AAA_g$ estimates $\E_{u \sim \DDD}[g(u)]$ within additive error
$\pm \tau$ with probability at least $1-\beta$.
\end{proof}

\paragraph{Simulation.} Lemma~\ref{lem:sim1} suggests a simple simulation of a nonadaptive (resp.\ adaptive) SQ algorithm by a noninteractive (resp.\ interactive) local algorithm as follows.
Assume the SQ algorithm makes at most $t$ queries to an SQ oracle $SQ_\DDD$. The local algorithm simulates each query $(g,\tau)$
by running $\AAA_g(n',\eps,LR_\z)$ with parameters $\beta'=\frac \beta t$ and
$n'= c \cdot \frac{\log(1/\beta')b^2}{\eps^{2}\tau^{2}}$ on a previously unused portion of the database
$\mathrm{z}$ containing $n'$ entries.

\begin{theorem}[Local simulation of SQ]\label{thm:local-sim-of-sq}
Let $\AAA_{\text{SQ}}$ be an SQ algorithm that makes at most $t$ queries to an SQ oracle $SQ_\DDD$, each with tolerance
at least $\tau$. The simulation above is $\eps$-differentially private.
If, for sufficiently large constant $c$, database $\z$ has $n\geq c \cdot \frac{t\log(t/\beta)b^2}{\eps^{2}\tau^{2}}$
entries sampled i.i.d.\ from the distribution $\DDD$ then the simulation above gives the same output as
$\AAA_{\text{SQ}}$  with probability at least $1-\beta$.

Furthermore, the
simulation is noninteractive if the original SQ algorithm $\AAA_{\text{SQ}}$ is nonadaptive.
The simulation is efficient if $\AAA_{\text{SQ}}$ is efficient.
\end{theorem}
\begin{proof}
Each query is simulated with a fresh portion of $\mathrm{z}$, and hence privacy is preserved as each entry
is subjected to a single application of the $\eps$-local randomizer $R$. By the union bound, the probability of
any of the queries not being approximated within additive error $\tau$ is bounded by $\beta$. If $\AAA_{\text{SQ}}$ is
nonadaptive, all queries to $LR_\z$ can be prepared in advance.
\end{proof}
\fi

\ifnum\full=1
\subsubsection{Simulation of Local Algorithms by SQ Algorithms} \label{sec:sq-local}
\else
\subsubsection{SQ Simulation of Local Algorithms} \label{sec:sq-local}
\fi
Let $\mathrm{z}$ be a database containing $n$ entries drawn i.i.d.\ from $\DDD$. Consider a local algorithm making $t$ queries to $LR_{\mathrm{z}}$. We show how to simulate any local randomizer invoked by this algorithm by using statistical queries to $SQ_{\DDD}$. Consider one such randomizer  $R: D \rightarrow W$ applied to database entry $z_i$.
%sofya: clarified; simulation of interactive algorithms is not explained in the short version anyway.
%To simulate $R$ we need to sample $w \in W$ with probability $p(w)= \Pr_{z_i \sim \DDD}[R(z_i) =~w \, |\, \mbox{conditioned on outputs given by}$ $\mbox{previous applications of randomizers to $z_i$}]$ taken over choice of $z_i \sim \DDD$ and random coins of $R$.
To simulate $R$ we need to sample $w \in W$ with probability $p(w)= \Pr_{z_i \sim \DDD}[R(z_i) =~w]$ taken over choice of $z_i \sim \DDD$ and random coins of $R$. (For interactive algorithms, it is more complicated, as the outputs of different randomizers applied to the same entry $z_i$ have to be correlated.)

\paragraph{A brief outline.} The idea behind the \ifnum\full=0 (omitted) \fi simulation is to
sample from a distribution $\widetilde{p}(\cdot)$ that is within small
statistical distance of $p(\cdot)$.
% This ensures a
%statistical distance of at most $\beta$ between the output
%distribution of the local algorithm and the distribution resulting
%from the simulation.
We start by applying $R$ to an arbitrary input (say, $\textbf{0}$) in the domain $D$
and obtaining a sample $w\sim R(\textbf{0})$. Let
$q(w)=\Pr[R(\textbf{0})=w]$ (where the probability is taken only over randomness in
$R$). Since $R$ is $\eps$-differentially private, $q(w)$ approximates
$p(w)$ within a multiplicative factor of $e^{\epsilon}$. To sample $w$
from $p(\cdot)$ we use the following rejection sampling algorithm: (i)
sample $w$ according to $q(\cdot)$; (ii) with probability
$\frac{p(w)}{q(w)e^\eps}$, output $w$; (iii) with the remaining probability, repeat from (i).

%Tocomplete the simulation, we show how statistical queries can be used to estimate $p(w)$.

To carry out this strategy, we must be able to estimate $p(w)$, which depends on the (unknown) distribution $\DDD$, using only SQ queries. The rough idea is to express $p(w)$ as the expectation, taken over $z\sim \DDD$, of the function $ h(z)=\Pr[R(z)=w]$ (where the probability is taken only over the coins of $R$). We can use $h$ as the basis of an SQ query. In fact, to get a sufficiently accurate approximation, we must rescale the function $h$ somewhat, and keep careful track of the error introduced by the SQ oracle. We present the details in the proof of the following lemma:

\begin{lemma} \label{lem:sim2-a}
%sofya: commented out D and n as they are not used; changed the query complexity of $\BBB$
%shiva-Reads wierd. So took half of old lemma and half of new lemma. Plus fixed the bug.
%Every local algorithm $\AAA$ with input $\mathrm{z}$, containing entries drawn i.i.d.\ from $\DDD$, can be simulated by an SQ algorithm $\BBB$. For every interactive (resp.\ noninteractive) local algorithm $\AAA$ making $t$ queries to $LR_{\mathrm{z}}$, the corresponding algorithm $\BBB$ is adaptive (resp.\ nonadaptive), in expectation uses $t\cdot e^{\eps}$ queries to $SQ_\DDD$ with tolerance $\tau = \Theta(\beta/t)$, and the statistical difference between $\BBB$'s and $\AAA$'s output distributions is at most $\beta$.
Let $\mathrm{z}$ be a database with entries drawn i.i.d.\ from a distribution $\DDD$.
For every noninteractive (resp.\ interactive) local algorithm $\AAA$ making $t$ queries to $LR_{\mathrm{z}}$,
there exists a nonadaptive (resp.\ adaptive) statistical query algorithm $\BBB$ that in expectation makes
$O(t\cdot e^{\eps})$ queries to $SQ_\DDD$ with accuracy $\tau = \Theta(\beta/(e^{2\eps}t))$, such that the
statistical difference between $\BBB$'s and $\AAA$'s output distributions is at most~$\beta$.
\end{lemma}
\ifnum\full=1
\begin{proof}
We split the simulation over Claims~\ref{claim:sim2-na} and \ref{claim:sim2-a}. In the first claim we simulate noninteractive local algorithms using nonadaptive SQ algorithms. In the second claim we simulate interactive local algorithms using adaptive SQ algorithms.

\begin{claim} \label{claim:sim2-na} For every noninteractive local algorithm $\AAA$ making $t$ nonadaptive queries to $LR_{\mathrm{z}}$,
there exists a nonadaptive statistical query algorithm $\BBB$ that in expectation makes $t\cdot e^{\eps}$ queries to $SQ_\DDD$ with accuracy $\tau = \Theta(\beta/(e^{2\eps}t))$, such that the statistical difference between $\BBB$'s and $\AAA$'s output distributions is at most $\beta$.
\end{claim}
\begin{proof}
We show how to simulate an $\eps$-local randomizer $R$ using
statistical queries to $SQ_{\DDD}$.
Because the local algorithm is
non-interactive, we can assume without loss of generality that it
accesses each entry $z_i$ only once. (Otherwise, one can combine
different operators, used to access $z_i$, by combining their answers
into a vector).  Given $R:D\rightarrow W$, we want to sample $w\in W$
with probability:
$$p(w)=\Pr_{z_i \sim \DDD}[R(z_i)=w].$$

Two notes regarding our notation: (i) As $z_i$ is drawn i.i.d.\ from $\DDD$ we could omit the index $i$. We leave the index $i$ in our notation to emphasize that we actually simulate the application of a local randomizer $R$ to entry $i$. (ii) The semantics of $\Pr$ changes depending on whether it appears with the subscript $z_i\sim\DDD$ or not. $\Pr_{z_i \sim \DDD}$ denotes probability that is taken over the choice of $z_i \sim \DDD$ and the randomness in $R$, whereas when the subscript is dropped $z_i$ is fixed and the probability is taken only over the randomness in $R$. Using this notation, $ \Pr_{z_i\sim\DDD}[R(z_i)=w]= \E_{z_i\sim\DDD} \Pr[R(z_i)=w]$.

We construct an algorithm $\BBB_{R,\eps}$ that given $t$, $\beta$,
and access to the SQ oracle, outputs $w\in W$, such that the statistical difference between the output probability distributions
of $\BBB_{R,\eps}$ and the simulated randomizer $R$ is at most $\beta/t$. Because the local algorithm makes $t$ queries, the overall
statistical distance between the output
distribution of the local algorithm and the distribution resulting
from the simulation is at most $\beta$, as desired.

\alg{An SQ algorithm $\BBB_{R,\eps}(t, \beta, SQ_\DDD)$ that simulates an
$\eps$-local randomizer $R:D\rightarrow W$.}
{\begin{enumerate}
\item Sample $w\sim R(\textbf{0})$. Let $q(w)=\Pr[R(\textbf{0})=w]$.
\item Define $g:D\to[-1,1]$ by $g(z_i)=\dfrac{\Pr[R(z_i)=w]-q(w)}{q(w)(e^\epsilon-e^{-\epsilon})}$,
  and let $\tau=\frac{\beta}{3e^{2\eps}t}$.
\item Query the SQ oracle $v=SQ_\DDD(g,\tau)$, and let
$\widetilde{p}(w)=vq(w)(e^\epsilon-e^{-\epsilon})+q(w)$.
\item With probability $\frac{\widetilde{p}(w)}{q(w)(1+\frac \beta {3t})e^\eps}$, output $w$.\\
With the remaining probability, repeat from Step 1.
\end{enumerate}}

We now show that the statistical distance between the output of $\BBB_{R,\eps}(t, \beta, SQ_\DDD)$ and the distribution $p(\cdot)$ is at most $\beta/t$.
As mentioned above, our initial approximation
$\widetilde{p}(\cdot)$ of $p(\cdot)$ in Step 1 is obtained by applying $R$ to some
arbitrary input (namely, $\textbf{0}$) in the domain $D$  and sampling $w\sim R(\textbf{0})$.  Since
$R$ is $\eps$-differentially private, $q(w)=\Pr[R(\textbf{0})=w]$ approximates $p(w)$
within a multiplicative factor of $e^{\epsilon}$.

However, to carry out the rejection sampling strategy, we need to get a much better estimate of $p(w)$. Steps 2 and 3 compute such an estimate, $\widetilde{p}(w)$, satisfying (with probability 1)
\begin{equation}
  \label{eq:tp}
  \widetilde{p}(w) \in \left(1\pm \phi\right)p(w)\quad\text{where}\quad \phi = \tfrac{\beta}{3t}\,.
\end{equation}

We establish the inclusion \eqref{eq:tp} below. For now, assume it holds on every iteration. Step 4 is a rejection sampling step which ensures that the output will follow a distribution close to $\widetilde{p}(\cdot)$. Inclusion \eqref{eq:tp} guarantees that $\frac{\widetilde{p}(w)}{q(w)(1+\frac \beta {3t})e^\eps}$ is at most 1, so the probability in Step 4 is well defined. The difficulty is that the quantity $\widetilde{p}(w)$ is not a well-defined function of $w$: it depends on the SQ oracle and may vary, for the same $w$, from iteration to iteration.

Nevertheless, $\widetilde{p}$ is fixed for any given iteration of the algorithm. In the given iteration, any particular element $w$ gets output with probability $q(w)\times \frac{\widetilde{p}(w)}{q(w)(1+\phi)e^\eps} = \frac{\widetilde{p}(w)}{(1+\phi)e^\eps}$.
The probability that the given iteration terminates (i.e., outputs some $w$) is then $p_{terminate}=\sum_w \frac{\widetilde{p}(w)}{(1+\phi)e^\eps}$. By \eqref{eq:tp}, this probability is in $\frac{1\pm \phi}{(1+\phi)e^\eps}$. Thus, \emph{conditioned on the iteration terminating}, element $w$ is output with probability $\frac{\widetilde{p}(w)}{(1+\phi)\cdot e^\eps\cdot p_{teminate}} \in \frac{1\pm\phi}{1\pm \phi}\cdot p(w)$. Since $\phi\leq 1/3$, we can simplify this to get
$$\Pr\left[w \text{ output in a given iteration} \big| \text{iteration produces output}\right] \in (1\pm 3\phi) p(w)\,.$$
This implies that no matter which iteration produces output, the statistical difference between the distribution of $w$ and $p(\cdot)$ will be at most $3\phi = \frac \beta t$, as desired.

Moreover, since each iteration terminates with probability at least $\frac{1-\phi}{1+\phi}\cdot e^{-\eps}$, the expected number of iterations is at most $\frac{1+\phi}{1-\phi} \cdot e^\eps \leq 2 e^\eps$. Thus, the total expected SQ query complexity of the simulation is $O(t\cdot e^{\eps})$.

It remains to prove the correctness of $\eqref{eq:tp}$. To estimate $p(w)$ given $w$,
 we set up the statistical query $g(z_i)$.
This is a valid query since $\Pr[R(z_i)=w]$ is a function of $z_i$, and furthermore $g(z_i) \in [-1,1]$ for all $z_i$ as $\Pr[R(z_i)=w]/\Pr[R(\mathbf{0})=w] \in e^{\pm \epsilon}$.
The SQ query result $v$ lies within $\E_{z_i \sim \DDD}[g(z_i)] \pm \tau $, where $\tau$ is the tolerance parameter for the statistical query, and so
$$\E_{z_i \sim \DDD}[g(z_i)] = \frac{\E_{z_i\sim\DDD}
  \Pr[R(z_i)=w]-q(w)}{q(w)(e^\epsilon-e^{-\epsilon})} =
\frac{p(w)-q(w)}{q(w)(e^\epsilon-e^{-\epsilon})}.$$
Plugging in the bounds for $v$ and $q(w)$ we get that
$\widetilde{p}(w) \in (1 \pm \tau')p(w)$ where $\tau'=e^{2\eps}\tau = \frac \beta {3t}$. This establishes \eqref{eq:tp} and concludes the proof.
%
% %shiva-explain that making t queries increases SD by factor of t.
% To conclude the proof, set $\tau' \leq \beta/t$ (where $t$ is the number of queries of algorithm $\AAA$). This guarantees that the statistical difference between distributions $p$ and $\widetilde{p}$ is at most $\beta/t$, and hence the statistical difference between $\BBB$'s and $\AAA$'s output distributions is at most~$\beta$.
% The expected number of SQ queries required to simulate each local randomizer via rejection sampling is $e^\eps$, and so the total expected query complexity of the simulation is $t \cdot e^\eps$.
\end{proof}

\begin{claim} \label{claim:sim2-a} For every interactive local algorithm $\AAA$ making $t$ queries to $LR_{\mathrm{z}}$, there exists an adaptive statistical query algorithm $\BBB$ that in expectation makes $O(t\cdot e^{\eps})$ queries $SQ_\DDD$ with accuracy $\tau = \Theta(\beta/(e^{2\eps}t))$, such that the statistical difference between $\BBB$'s and $\AAA$'s output distributions is at most $\beta$.
\end{claim}
\begin{proof}
  As in the previous claim, we show how to simulate the output of the
  local randomizers during the run of the local algorithm. A
  difference, however, is that because an entry $z_i$ may be accessed
  multiple times, we have to condition our sampling on the outcomes of
  previous (simulated) applications of local randomizers to
  $z_i$. 

  More concretely, let $R_1,R_2,...$ be the sequence of randomizers that access the entry $z_i$. To simulate $R_k(z_i)$, we must take into account the answers $a_1,\ldots,a_{k-1}$ given by the simulations of % from the $k-1$  previous simulations of
 $R_1(z_i),\ldots,R_{k-1}(z_i)$. 
% %
% % consider the $k$th time that $z_i$ is accessed by one of the randomizers, and let $R_k$ denote this randomizer. 
% When we
%   simulate $R_k(z_i)$ we have to sample from the distribution
%   conditioned on the answers $a_1,\ldots,a_{k-1}$ given by the simulations of % from the $k-1$  previous simulations of
%  $R_1(z_i),\ldots,R_{k-1}(z_i)$. \rnote{Is
%     this clear?}\anote{tried to clarify. old version in tex comments.}
% % When we
% %   simulate $R_k(z_i)$ we have to condition the distribution
% %   over the answers $a_1,\ldots,a_{k-1}$ resulting from our
% %   previous  $k-1$ simulations of $R_1(z_i),\ldots,R_{k-1}(z_i)$.
%
We  show how to do this
%simulate the output of $R_k$ 
using adaptive statistical queries to $SQ_{\DDD}$. 
The notation is the same as in Claim~\ref{claim:sim2-na}. We want to output $w\in W$ with probability $$p(w) = \Pr_{z_i \sim \DDD}[R_k(z_i) = w \, | \, R_{k-1}(z_i)=a_{k-1},R_{k-2}(z_i)=a_{k-2},\dots,R_{1}(z_i)=a_{1}], $$
where $R_{j}$ ($1 \leq j \leq k-1$) denotes the $j$th randomizer applied to $z_i$.

As before, we start by sampling $w\sim R(\textbf{0})$. Let $q(w)=\Pr[R_k(\textbf{0})=w]$. Note that $q(w)$ approximates $p(w)$ within a multiplicative factor of $e^{\epsilon}$ because $R_1,\ldots,R_k$ are respectively $\eps_1$-,$\ldots,\eps_k$-differentially private, and $\eps_1+\ldots+\eps_k \leq \eps$. Hence, we can use the rejection sampling algorithm as in Claim~\ref{claim:sim2-na}. Rewrite~$p(w)$:
\begin{eqnarray*}
p(w)&=&\frac{\Pr_{z_i \sim \DDD}[R_k(z_i)=w \wedge R_{k-1}(z_i)=a_{k-1} \wedge \dots \wedge R_1(z_i)= a_1]}{\Pr_{z_i \sim \DDD}[R_{k-1}(z_i)=a_{k-1} \wedge \dots \wedge R_1(z_i)= a_1]} \\
&=&\frac{\E_{z_i \sim \DDD}[\Pr[R_k(z_i)=w \wedge R_{k-1}(z_i)=a_{k-1} \wedge\dots \wedge R_{1}(z_i)=a_1]] }{\E_{z_i \sim \DDD}[\Pr[R_{k-1}(z_i)=a_{k-1} \wedge \dots \wedge R_{1}(z_i)=a_1]]} 
\end{eqnarray*}
Conditioned on a particular value of $z_i$, the probabilities in the last expression depend only the coins of the randomizers. The outputs of the randomizers are independent conditioned on $z_i$, and therefore we can simplify the expression above:
$$
p(w)=\frac{\E_{z_i \sim \DDD}\left [\Pr[R_k(z_i)=w] \cdot \prod_{j=1}^{k-1}\Pr[R_j(z_i)=a_j]\right ]}{\E_{z_i \sim \DDD}\left [\prod_{j=1}^{k-1}\Pr[R_j(z_i)=a_j] \right]}\label{eq:pw2}
$$
% The transition from Equation~(\ref{eq:pw1}) to
% Equation~(\ref{eq:pw2}) is possible because the probabilities are taken only
% over the independent randomness of the local randomizers.
Let $p_1$ and $p_2$ denote the numerator and denominator, respectively, in the right hand side of the equation above.
Let $r_1(z_i)$ and $r_2(z_i)$ denote the values inside the expectations that define $p_1$ and $p_2$, respectively. Namely,  $$r_1(z_i)= \Pr[R_k(z_i)=w] \cdot
\prod_{j=1}^{k-1}\Pr[R_j(z_i)=a_j] \ \ \ \ \mbox{ and }
\ \ \ \ r_2(z_i)=\prod_{j=1}^{k-1}\Pr[R_j(z_i)=a_j]\,.$$ For estimating
$p_1 = \E_{z_i\sim\DDD}[ r_1(z_i)]$ we use the statistical query
$g_1(z_i)$, and for estimating $p_2 = \E_{z_i\sim\DDD} [r_2(z_i)]$ we use
the statistical query $g_2(z_i)$ defined as follows:
$$g_1(z_i)=\frac{r_1(z_i)-r_1(\textbf{0})}{r_1(\textbf{0})(e^\eps-e^{-\eps})} \; \ \ \ \ \mbox{ and } \;  \ \ \ \ g_2(z_i)=\frac{r_2(z_i)-r_2(\textbf{0})}{r_2(\textbf{0})(e^\eps-e^{-\eps})}.$$
As in Claim~\ref{claim:sim2-na}, one can estimate $p_1$ and $p_2$ to within a multiplicative factor of $(1\pm\tau')$ where $\tau' = e^{2\eps}\tau$ and $\tau$ is the accuracy of the statistical queries. The ratio of the estimates for $p_1$ and $p_2$ gives an estimate $\tilde p(w)$ for $p(w)$ to within a multiplicative factor $ (1\pm 3\tau')$, for $\tau' \leq \frac13$. The estimate $\tilde p(w)$ can then be used with rejection sampling to sample an output of the randomizer.

Let $t$ be the number of queries made by $\AAA$. Setting $\tau' \leq \frac\beta{3t}$ guarantees that the statistical difference between distributions $p$ and $\widetilde{p}$ is at most $\frac \beta t$, and hence the statistical difference between $\BBB$'s and $\AAA$'s output distributions is at most~$\beta$. As in Claim~\ref{claim:sim2-na}, the expected number of SQ queries for rejection sampling is $O(t \cdot e^\eps)$.
\end{proof}
%shiva-added this line to the full version. Again I believe this should be 2 lemma's rather than 1 lemma with two sub-claims inside it.
Claims~\ref{claim:sim2-na} and~\ref{claim:sim2-a} imply Lemma~\ref{lem:sim2-a}.
\end{proof}
\fi

\ifnum\full=1  %shiva-need to add more here for the full version.
Note that the efficiency of the constructions in Lemma~\ref{lem:sim2-a} depends on the efficiency of computing the functions submitted to the SQ oracle, \eg the efficiency of computing the probability $\Pr[R(z_i)=w]$. We discuss this issue in the next section.
\fi

\subsection{Implications for Local Learning}\label{sec:locallearn}
%{SQ Learning \texorpdfstring{$\equiv$}{=} Local Learning} \label{app:local=sq}
%Sofya: changed defs below (local learning did not mention eps)
In this section, we define learning in the local and SQ models. The equivalence of the two models follows from the simulations described in the previous sections. An immediate but important corollary is that local learners are strictly less powerful than general private learners.

\begin{definition}[Local Learning] \emph{Locally learnable}
is defined identically to privately PAC learnable (\defref{private-general}),
except for the additional requirement that for all $\eps>0$, algorithm $\mathcal{A(\eps,\cdot,\cdot,\cdot)}$ is $\epsilon$-local and invokes $LR_{\mathrm{z}}$ at most $poly(d,\mathrm{\it size}(c),
1/\eps,1/\alpha,\log(1/\beta))$ times.
\ifnum\full=1
Class $\CCC$ is efficiently locally learnable if both: (i) the running time of $\mathcal{A}$ and (ii) the time to evaluate each query that $\AAA$ makes are bounded by some polynomial in $d,\mathrm{\it size}(c),
1/\eps, 1/\alpha$, and $\log(1/\beta)$.
\fi
\end{definition}

%shiva-added this line. In the full version make it a defn.
Let $\XXX$ be a distribution over an input domain~$X$. Let $SQ_{c,\XXX}$ denote the statistical query oracle that takes as input a function $g: X \times \{+1,-1\} \rightarrow \{+1,-1\}$ and a tolerance parameter $\tau \in (0,1)$ and outputs $v$ such that: $|v-\E_{x \sim \XXX}[g(x,c(x))]| \leq \tau$.

\begin{definition}[SQ Learning\footnote{\label{foot:sqerror}The standard definition of SQ learning does not allow for any probability of error in the learning algorithm (that is, $\beta=0$). Our definition allows for a small failure probability $\beta$. This enables cleaner equivalence statements and clean modeling of randomized SQ algorithms. One can show that differentially private algorithms must have some non-zero probability of error, so a relaxation along these lines is necessary for our results.}]\label{def:sq}
%A concept class $\CCC$ over $X$ is {\em SQ learnable} using hypothesis class $\mathcal{H}$ if there exists an algorithm $\mathcal{A}$ and a polynomial $poly(\cdot,\cdot,\cdot,\cdot)$ such that for all $d \in \N$, all concepts $c \in \CCC_d$, all distributions $\XXX$ on $X_d$, and all $\alpha,\beta \in (0,1/2)$, given access to $SQ_{c,\XXX}$ and inputs $\alpha, \beta$, algorithm $\mathcal{A}$ with probability at least $1-\beta$ outputs a hypothesis $h \in \mathcal{H}$ satisfying $\trueerror(h)\leq \alpha$.
{\em SQ learnable} is defined identically to PAC learnable (\defref{PAC}),
except that instead of having access to examples $\z$, an SQ learner
\ifnum\full=1
$\mathcal{A}$
\fi
can  make $poly(d,\mathrm{\it size}(c),
1/\alpha,\log(1/\beta))$ queries to oracle $SQ_{c,\XXX}$ with tolerance $\tau\geq 1/poly(d,\mathrm{\it size}(c),
1/\alpha,\log(1/\beta))$.
\ifnum\full=1
Class $\CCC$ is efficiently $SQ$ learnable if both: (i) the running time of $\mathcal{A}$ and (ii) the time to evaluate each query that $\AAA$ makes are bounded by some polynomial in $d,%\mathrm{\it size}(c),
1/\alpha$, and $\log(1/\beta)$.
\fi
\end{definition}

\ifnum\full=0
From the simulations in Section~\ref{sec:local-sq} and \ref{sec:sq-local} we obtain the equivalence between SQ and local learning:
%sofya: moved the following sentence to the section, not specific to learning
%In the full version of this paper~\cite{KLNRS08}, we also investigate the efficiency of these simulations.

%shiva-reworded the theorem below for a referee comment
%sofya: what was the comment?
%changed the theorem
%\begin{theorem} \label{thm:nasq-nilr}
%Let $\CCC$ be a concept class over $X$. Let $\XXX$ be a distribution over $X$. Let $z=(z_1,\dots,z_n)$ denote a database where every $z_i=(x_i,c(x_i))$ with $x_i$ drawn i.i.d.\ from $\XXX$ and $c \in \CCC$. Concept class $\CCC$ is locally learnable using $\HHH$ by an interactive (resp.\ noninteractive) local learner  with inputs $\alpha,\beta$, and with access to $LR_{\mathrm{z}}$ \emph{if and only if} $\CCC$ is SQ learnable using $\HHH$ by an adaptive (resp.\ nonadaptive) SQ learner with inputs $\alpha,\beta$, and access to $SQ_{c,\XXX}$.
%\end{theorem}
\begin{theorem} \label{thm:nasq-nilr}
A concept class is learnable by a noninteractive (resp.\ interactive) local learner \emph{if and only if} it is learnable  by a nonadaptive (resp.\ adaptive) SQ learner.
\end{theorem}
\fi

\ifnum\full=1

In order to state the equivalence between SQ and local learning, we require the following
efficiency condition for a local randomizer.

\begin{definition}[Transparent Local Randomizer] \label{def:transparent} Let $R:D \rightarrow W$ be an $\eps$-local randomizer. The randomizer is {\em transparent} if both: (i) for all inputs $u \in D$, the time needed to evaluate $R$; and (ii) for all inputs $u \in D$ and outputs $w\in W$ the time taken to compute the probability $\Pr[R(u)=w]$, are polynomially bounded in the size of the input and $1/\eps$.
\end{definition}

As stated, this definition requires {\em exact} computation of probabilities. This may not make sense on a finite-precision machine, since for many natural randomizers the transition probabilities are irrational. One can relax the requirement to insist that relevant probabilities are computable with additive error at most $ \phi$ in time polynomial in $\log(\frac 1 \phi)$.

All local protocols that have appeared in the
literature~\cite{EGS03,AH05,AS00,AA01,EGS03,MS06,HB08} are
transparent, at least in this relaxed sense.

In the equivalences of the previous sections,  \emph{transparency} of local randomizers corresponds directly to  \emph{efficient computability} of the function $g$ in an SQ query. To see why, consider first  the simulation of SQ algorithms by local algorithms:
if the original SQ algorithm is efficient (that is, query $g$ can be
evaluated in polynomial time) then the local randomizer $R(u)=g(u) +
\eta$ can also be evaluated in polynomial time for all $u \in
D$. Furthermore, it is simple to estimate for all inputs $u \in D$ and
outputs $w\in W$ the probability $\Pr[R(u)=w]$ since $R(u)$ is a Laplacian random
variable with known parameters. Second, in the SQ simulation of a local algorithm, the functions $g(z_i)=\frac{\Pr[R(z_i)=w]-q(w)}{q(w)(e^\epsilon-e^{-\epsilon})}$ that are constructed can be evaluated efficiently precisely when the local randomizers are transparent.

We can now state the main result of this section, which follows from Lemmas~\ref{lem:sim1} and~\ref{lem:sim2-a}, along
with the correspondence  between transparent randomizers and efficient SQ queries.

\begin{theorem}\label{thm:nasq-nilr}
Let $\CCC$ be a concept class over $X$. Let $\XXX$ be a distribution over $X$. Let $z=(z_1,\dots,z_n)$ denote a database where every $z_i=(x_i,c(x_i))$ with $x_i$ drawn i.i.d.\ from $\XXX$ and $c \in \CCC$. Concept class $\CCC$ is locally learnable using $\HHH$ by an interactive local learner  with inputs $\alpha,\beta$, and with access to $LR_{\mathrm{z}}$ \emph{if and only if} $\CCC$ is SQ learnable using $\HHH$ by an adaptive SQ learner with inputs $\alpha,\beta$, and access to $SQ_{c,\XXX}$.

Furthermore, the simulations guarantee the following additional properties:
(i)  an \emph{efficient} SQ learner is simulatable by an \emph{efficient} local learner that uses only \emph{transparent} randomizers;
(ii) an \emph{efficient} local  learner  that uses only \emph{transparent} randomizers is simulatable by an \emph{efficient} SQ learner;
(iii) a \emph{nonadaptive} SQ (resp.\ noninteractive local) learner is simulatable by a \emph{noninteractive} local (resp.\ nonadaptive SQ) learner.
\end{theorem}
\fi

Now we can use lower bounds for SQ learners for \parity\ (see, \eg
\cite{Kearns98,BFJKMR-1994,Yang-2002}) to demonstrate limitations of
local learners. The lower bound of \cite{BFJKMR-1994} rules out SQ
learners for \parity\ that use at most $2^{d/3}$ queries of tolerance
at least $2^{-d/3}$, even (a) allowing for unlimited computing time,
(b) under the restriction that examples be drawn from the uniform
distribution and (c) allowing a small probability of error (see
Footnote~\ref{foot:sqerror}). Since \parity\ is (efficiently)
privately learnable (\thmref{parity-in-PPAC}), and since local
learning is equivalent to SQ learning, we obtain:
\begin{corollary}
\label{cor:sep1}
Concept classes learnable by local learners are a strict subset of concept classes PAC learnable privately. This holds both with and without computational restrictions.
\end{corollary}

\subsection{The Power of Interaction in Local Protocols}\label{sec:mparity}

To complete the picture of locally learnable concept classes, we consider how interaction changes the power of local learners (and, equivalently, how adaptivity changes SQ learning). \rnote{Removed sentence fragment; do we want to say more?}
%adaptivity gives more power in the SQ model. 
As mentioned in the introduction, interaction is very costly in typical applications of local algorithms.
We show that this cost is sometimes necessary,  by giving a concept class that an interactive algorithm can learn efficiently with a polynomial number of examples drawn from the uniform distribution, but for which any noninteractive algorithm requires an exponential number of examples under the same distribution.

Let \mparity\ be the class of functions $c_{r,a}:\zo^{d} \times \zo^{\log d} \times \zo \rightarrow \oo$ indexed by $r\in\zo^d$ and $a\in \zo$:
$$
c_{r,a}(x,i,b)=
\begin{cases}
(-1)^{ r \odot x +a} & \text{if } b=0,\\
(-1)^{r_i} & \text{if $b=1$,}
\end{cases}
$$
where $ r \odot x $ denotes the inner product of $r$ and $x$ modulo $2$, and $r_{i}$ is the $i$th bit of $r$.
This concept class divides the domain into two parts (according to the last bit, $b$). When $b=0$, the concept $c_{r,a}$ behaves either like the $\parity$ concept indexed by $r$, or like its negation, according to the bit $a$ (the ``mask'').
When $b=1$, the concept essentially ignores the input example and outputs some bit of the parity vector $r$.

Below, we consider the learnability of  $\mparity=\{c_{r,a}\}$ when the examples are drawn from the uniform distribution  over the domain $\{0,1\}^{d+\log d+1}$. In \secref{admparity}, we give a  {\em adaptive} SQ learner for $\mparity$ under the uniform distribution. The adaptive learner uses two rounds of communication with the SQ oracle: the first, to learn $r$ from the $b=1$ half of the input, and the second, to retrieve the bit $a$ from the $b=0$ half of the input via queries that depend on $r$.

In \secref{nasqnotsq}, we show that no \emph{nonadaptive} SQ learner which uses $2^{o(d)}$ examples can consistently produce a hypothesis that labels significantly more than  $3/4$ of the domain correctly. 
The intuition is that as the queries are prepared nonadaptively, any information about $r$ gained from the $b=1$ half of the inputs cannot be used to prepare queries to the $b=0$ half. Since information about $a$ is contained only in the $b=0$ half, in order to extract $a$, the SQ algorithm is forced to learn \parity, which it cannot do with few examples. Our separation in the SQ model directly translates to a separation in the local model (using Theorem~\ref{thm:nasq-nilr}).

The following theorem summarizes our results.

% \thmref{nasqnotsq} below summarizes our results. In order to state it precisely, we need some addition terminology.

\begin{theorem}\label{thm:nasqnotsq}{~}
  \begin{enumerate}
  \item There exists an efficient adaptive  SQ learner for
    $\mparity$ over the uniform distribution.

  \item No nonadaptive SQ learner can  learn $\mparity$ (with a
    polynomial number of queries) even under the uniform
    distribution on examples. Specifically, there is an SQ oracle $\cal O$ such that any nonadaptive SQ learner that
    makes $t$ queries to $\mathcal{O}$ over the uniform distribution, all with tolerance at least $2^{-d/3}$,  satisfies
    the following: if the concept $c_{\bar{r},\bar{a}}$ is drawn
    uniformly at random from the set of $\mparity$ concepts, then, with probability at least $\frac
    1 2 - \frac t {2^{d/3+2}}$ over $c_{\bar r, \bar a}$, the output hypothesis $h$ of the
    learner has $\trueerror(c_{\bar{r},\bar{a}},h) \geq \frac 1 4$.
  \end{enumerate}
\end{theorem}

\ifnum\full=1
\begin{corollary}
\label{cor:sep2}
The concept classes learnable by nonadaptive SQ learners (resp.\ noninteractive local learners) under the uniform distribution are a strict subset of the concept classes learnable by adaptive SQ learners (resp.\ interactive local learners) under the uniform distribution. This holds both with and without computational restrictions.
\end{corollary}
\fi

\paragraph{Weak vs. Strong Learning.}
% We say that a (SQ) learning algorithm is a {\em weak} (SQ) learning algorithm if it produces a hypothesis whose error on the target concept is noticeably less than 1/2 (and not necessarily any $\epsilon > 0$). More precisely, weak learning (SQ) algorithm produces a hypothesis $h$ that with probability at least $1/poly(d,\mathrm{\it size}(c))$ satisfies $\trueerror(c,h) \leq 1/2-1/poly(d,\mathrm{\it size}(c))$ for a fixed polynomial $poly(\cdot,\cdot)$. To distinguish from this notion we will refer to our original definition of SQ learning (Definition~\ref{def:sq}) as {\em strong} SQ learning.

The learning theory literature distinguishes between {\em strong} learning, in which the learning algorithm is required to produce hypotheses with arbitrarily low error (as in \defref{PAC}, where the parameter $\alpha$ can be arbitrarily small), and {\em weak} learning, in which the learner is only required to produce a hypothesis with error bounded below $1/2$ by a polynomially small margin. The separation proved in this section (\thmref{nasqnotsq}) applies only to {\em strong} learning: although no nonadaptive SQ learner can produce a hypothesis with error much better than $1/4$,  it is simple to design a nonadaptive weak SQ learner for $\mparity$ under the uniform distribution with error exactly 1/4.

In fact, it is impossible to obtain an analogue of our separation for weak  learning.  The characterization of SQ learnable classes in terms of ``SQ dimension'' by Blum \etal\cite{BFJKMR-1994} implies that adaptive and nonadaptive SQ algorithms are equivalent for weak learning.
This
is not explicit in \cite{BFJKMR-1994}, but follows from the fact that the weak learner constructed for classes with low SQ dimension is non-adaptive. (Roughly, the learner works by checking if the concept at hand is approximately equal to one of a polynomial number of alternatives; these alternatives depend on the input distribution and the concept class, but not on the particular concept at hand.)
% NOT TRUE:
% Because boosting (see, \eg \cite{Kearns94}) allows one to convert any distribution-free weak learner into a distribution-free strong learner in a black-box fashion\anote{check that privacy does not interfere with this}, one also gets that no separation is possible between {\em distribution-free} adaptive and nonadaptive SQ learning.
%shiva-8/3/08: Remarkably, I found out from a recent paper {On The Power of Membership Queries in Agnostic Learning) that membership queries doesn't help in distribution free agnostic learning. So privacy modeling with membership queries is not that interesting.
%ads 2/24/09: Following up on Homin's question: boosting is indeed, itself, adaptive. So the remark above is not true.

\paragraph{Distribution-free vs Distribution-specific Learning} The results of this section concern the learnability of $\mparity$ under the uniform distribution. The class $\mparity$ does not separate adaptive from nonadaptive {\em distribution-free} learners, since $\mparity$ cannot be learned by any SQ learner under the distribution which is uniform over examples with $b=0$ (in that case, learning $\mparity$ is equivalent to learning $\parity$ under the uniform distribution). Separating adaptive from nonadaptive {\em distribution-free} SQ learning remains an open problem.\anote{Is it really open?}

\subsubsection{An Adaptive Strong SQ Learner for \texorpdfstring{\mparity}{MASKED-PARITY} over the Uniform Distribution} \label{sec:admparity}
Our adaptive learner for $\mparity$ uses two rounds of communication with the SQ oracle: first, to learn $r$ from the $b=1$ half of the input, and second, to retrieve the bit $a$ from the $b=0$ half of the input via queries that depend on $r$. \thmref{nasqnotsq}, part (1), follows from the proposition below.
\alg{Adaptive SQ Learner $\AAA_{\sf MP}$ for $\mparity$ over the Uniform Distribution}{
\begin{enumerate}
\item For $j=1,\dots, d$ (in parallel)
\begin{enumerate}
\item Define $g_j:D\to \zo$ by $$g_j(x,i,b,y) = (i=j)\;\wedge\;(b=1)\;\wedge\;(y=-1)\, ,$$
where   $x\in\zo^d$, $i\in\zo^{\log d}$, $b\in\zo$, and $y=c_{r,a}(x,i,b)\in\oo$.
\item  $answer_j\gets SQ_\DDD(g_j,\tau),$ where $\tau=\frac 1 {4d+1}$, and
  $\hat{r}_j \gets
\begin{cases}
1 & \text{if } answer_j>\frac 1 {4d};\\
0 & \text{otherwise.}
\end{cases}$
\end{enumerate}
\item
\begin{enumerate}
\item  $\hat{r}\gets \hat{r_1}\dots\hat{r_d}\in\zo^d$

\item Define $g_{d+1}:D\to \zo$ by $$g_{d+1}(x,i,b,y)= (b=0)\;\wedge\;(y\not=(-1)^{ \hat{r} \odot x })\,.$$
where $x\in\zo^d$, $i\in\zo^{\log d}$, $b\in\zo$, and $y=c_{r,a}(x,i,b)\in\oo$.
\item  $answer_{d+1}\gets SQ_\DDD(g_{d+1},\frac 1 5).$, and  $\hat{a} \gets
\begin{cases}
1 & \text{if } answer_{d+1}>\frac 1 {4};\\
0 & \text{otherwise.}
\end{cases}$
\item Output $c_{\hat{r},\hat{a}}.$
\end{enumerate}
\end{enumerate}
}

\begin{prop}[\thmref{nasqnotsq}, part (1), in detail]
  The algorithm $\AAA_{\sf MP}$  {\em efficiently} learns $\mparity$ (with probability 1)  in 2 rounds using $d+1$ SQ queries computed over the uniform distribution with minimum tolerance $\frac 1{4d+1}$.
\end{prop}

\begin{proof}
  Consider the $d$ queries in the first round. If $r_j=1$, then
$$\E_{(x,i,b,y)\gets\DDD}[g_j(x,i,b,y)]=\Pr_{i\in_u\{0,1\}^{\log d}, b\in_u\{0,1\}}
[(i=j)\;\wedge\;(b=1)]=\frac{1}{2d}\,.$$ If $r_j=0$, then
$\E[g_j(x,i,b,y)]=0$. %(the probability is taken over the random choice of $(x,i,b)$ from $\DDD$).
Since the tolerance $\tau$ is less than $\frac 1 {4d}$, each query $g_j$
reveals the $j$th bit of $r$ exactly. Thus, the estimate $\hat{r}_j$ is exactly $r_j$, and $\hat r=r$.

Given that $\hat r$ is correct, the second round query $g_{d+1}$ is always 0 if $a=0$. If $a=1$, then $g_{d+1}$ is 1 exactly when $b=0$. Thus  $\E[g_{d+1}(x,i,b,y)]=\frac a 2$ (where $a\in\zo$). Since the tolerance is less than $\frac 1 4$, querying $g_{d+1}$ reveals~$a$: that is, $\hat{a}=a$, and so the algorithm outputs the target concept.
%The two stage adaptive SQ learner is:
%
%\begin{CompactEnumerate}
%\item Let $y=c_{r,a}(x,i,b)$. For $0 \leq j < d$ define $g_j(x,i,b,y)$, where $x\in\zo^d$, $i\in\zo^{\log d}$, $b\in\zo$, $y\in\oo$, to be the function $(j=i)\;\wedge\;(b=1)\;\wedge\;(y=-1)$.
%Then $\E[g_j(x,i,b,y)]=\frac{1}{2d}$ if $r_j=1$ and,
%otherwise, $\E[g_j(x,i,b,y)]=0$ (the probability is taken over the random choice of $(x,i,b)$ from $\DDD$).
%Hence, using queries $g_0,\ldots,g_{d-1}$ and setting tolerance $\tau < 1/4d$ allows us to learn $r$ exactly.
%
%\item Once we know $r$, we define the function $g_d(x,i,b,y)$ to be $(b=0)\;\wedge\;(y\not=(-1)^{ r \odot x })$. Then $\E[g_d(x,i,b,y)]=a/2$, and hence by issuing the query $g_d$ with error parameter less than $1/4$, we learn~$a$.
%\end{CompactEnumerate}

Note that the functions $g_1,\ldots,g_{d+1}$ are all computable in time $O(d)$, and the computations performed by $\AAA_{\sf MP}$ can be done in time $O(d)$, so the SQ learner is efficient. %Therefore, $\mparity$ is efficiently interactively locally learnable when the entries of the database are drawn from the uniform distribution (Theorem~\ref{thm:nasq-nilr}).
\end{proof}

% \ifnum\full=0
% We consider the concept class $\mparity=\{c_{r,a}\}$ when the underlying distribution $\XXX$ is uniform over binary strings of length $d+\log d+1$.
% Our adaptive learner for $\mparity$ uses two rounds of communication with the SQ oracle: first, to learn $r$ from the $b=1$ half of the input, and second, to retrieve the bit $a$ from the $b=0$ half of the input via queries that depend on $r$.
% The impossibility result (Theorem~\ref{thm:nasqnotsq}) for nonadaptive learners uses ideas from statistical query lower bounds (see, \eg \cite{Kearns98,BFJKMR-1994,Yang-2002}). The intuition is that as the queries are prepared nonadaptively, any inforddmation about $r$ gained from the $b=1$ half of the input cannot be used to prepare queries to the $b=0$ half. Since information about $a$ is contained only in the $b=0$ half, in order to extract $a$, the SQ algorithm is forced to learn \parity.  Our separation in the SQ model directly translates to a separation in the local model (using Theorem~\ref{thm:nasq-nilr}).\fi

\subsubsection{Impossibility of non-adaptive SQ learning for $\mparity$}
\label{sec:nasqnotsq}

The impossibility result (\thmref{nasqnotsq}, part (2)) for nonadaptive learners uses ideas from statistical query lower bounds (see, \eg \cite{Kearns98,BFJKMR-1994,Yang-2002}). 

\begin{proof}[Proof of \thmref{nasqnotsq}, part (2)]
Recall that the distribution $\DDD$ is uniform over $D=\zo^{d + \log(d)+1}$. For functions $f,h:\zo^{d+\log d + 1}\rightarrow\oo$, recall that $\trueerror(f,h) =\Pr_{x \sim \DDD}[f(x)\not=h(x)]$. Define the inner product of $f$ and $h$ as:
$$\langle f,h\rangle = \frac{1}{|D|}\sum_{x\in D}f(x)h(x) = \E_{x \sim \DDD}[f(x)h(x)].$$
% was: $$\langle f,h\rangle = \frac{1}{2^{d+\log d +1}}\sum_{x\in\zo^{d+\log d +1}}f(x)h(x) = \E[f(x)h(x)].$$
The quantity $\langle f,h\rangle=\Pr_{x \sim \DDD}[f(x)=h(x)]-\Pr_{x \sim \DDD}[f(x)\not=h(x)]=1-2\cdot \trueerror(f,h)$ measures the correlation between $f$ and $h$ when $x$ is drawn from the uniform distribution~$\DDD$.

Let the target function $c_{\bar{r},\bar{a}}$ be chosen uniformly at random from the set $\{c_{r,a}\}$. Consider a nonadaptive SQ algorithm that makes $t$ queries $g_1,\dots,g_t$.
\anote{minor edits based on Sofya's comments.}
%While the learner's final hypothesis need not be independent of $\bar{r}$ and $\bar{a}$,
The queries $g_1\through g_t$ \emph{must be independent of $\bar{r}$ and $\bar{a}$ since the learner is nonadaptive}. The only information about $\bar{a}$ is in the outputs associated with the $b=0$ half of the inputs (recall that $c_{\bar{r},\bar{a}}(x,i,b)=(-1)^{r_i}$ when $b=1$).

The main technical part of the proof follows the lower bound on SQ learning of $\parity$. Using Fourier analysis, we split the true answer to a query into three components: a component that depends on the query $g$ but not the pair $(\bar{r},\bar{a})$, a component that depends on $g$ and $\bar{r}$ (but not $\bar a$), and a component that depends on $g,\bar{r}$, and $\bar{a}$ (see Equation (\ref{eq:3terms}) below). We show that for most target concepts $c_{\bar{r},\bar{a}}$ the last component can be ignored by the SQ oracle. That is, a very close approximation to the correct output to the SQ queries made by the learner can be computed solely based on $g$ and $\bar{r}$.  Consequently, for most target concepts $c_{\bar r,\bar a}$, the SQ oracle can return answers that are independent of $\bar{a}$, and hence $\bar{a}$ cannot be learned.

% was: Following the lower bound on SQ learning of $\parity$, we decompose the queries in the Fourier basis over $\zo^{d+\log d+1}$ and show that for most target concepts, the exactly correct outputs to the SQ queries all lie very close to their average value, taken over all concepts. Hence, for most target concepts $c_{\bar r,\bar a}$, the SQ oracle can simply return constants independent of $\bar{r}$ and $\bar{a}$.
\snote{Made changes to this para and following equations.} \rnote{Needs editing}
Consider a statistical query $g:\zo^d \times \zo^{\log d} \times \zo \times \oo \rightarrow \oo$. 
For some $(x,i,b)\in D$, the value of $g(x,i,b,\cdot)$ depends on the label (i.e., $(g(x,i,b,+1)\not=g(x,i,b,-1))$) and otherwise $g(x,i,b,\cdot)$ is insensitive to the label (i.e., $(g(x,i,b,+1)=g(x,i,b,-1))$). 
Every statistical query $g(\cdot,\cdot,\cdot,\cdot)$ can be decomposed into a label-independent and 
label-dependent part. This fact was first implicitly noted by Blum~\emph{et al.}~\cite{BFJKMR-1994} 
and made explicit by Bshouty and Feldman~\cite{BshoutyFeldman-2001} (Lemma 30). 
We adapt the proof presented in~\cite{BshoutyFeldman-2001} for our purpose.

Let $$f_g(x,i,b) = \frac{g(x,i,b,1)-g(x,i,b,-1)}{2} \qquad \text{and}\qquad  C_g = \frac{1}{2}\E[g(x,i,b,1)+g(x,i,b,-1)]\, .$$
We can rewrite the expectation of $g$ on any concept $c_{\bar{r},\bar{a}}$ in terms of these quantities:\anote{cite this from [17]?}
$$\E[g(x,i,b,c_{\bar{r},\bar{a}}(x,i,b))] =  C_g + \langle f_g, c_{\bar{r},\bar{a}} \rangle\,.$$
% \begin{eqnarray*}
% \lefteqn{\E[g(x,i,b,c_{\bar{r},\bar{a}}(x,i,b))]} \\  & =& \E\left[g(x,i,b,-1)\frac{1-c_{\bar{r},\bar{a}}(x,i,b)}{2} + g(x,i,b,1)\frac{1+c_{\bar{r},\bar{a}}(x,i,b)}{2} \right ] \\
% & = & \frac{1}{2}\E[g(x,i,b,1)c_{\bar{r},\bar{a}}(x,i,b)] - \frac{1}{2}\E[g(x,i,b,-1)c_{\bar{r},\bar{a}}(x,i,b)] + \frac{1}{2} \E[g(x,i,b,1)] +\frac{1}{2}\E[g(x,i,b,-1)] \\
% & = &  \frac{1}{2}\E[(g(x,i,b,1)-g(x,i,b,-1))c_{\bar{r},\bar{a}}(x,i,b)] + C_g \\
% % \ \ \ \ \hfill{(\mbox{where } C_g = \frac{1}{2}\E[g(x,i,b,1)+g(x,i,b,-1)])} \\
% & = &\frac{1}{2}\E[2 \cdot f_g(x,i,b) \cdot c_{\bar{r},\bar{a}}(x,i,b)]+C_g \\
% % \ \ \ \ \ \ \ \ \ \ \ \ \ \ \ \ \ \ \ \ \ \ \hfill{ \left (\mbox{where } f_g(x,i,b) = \frac{g(x,i,b,1)-g(x,i,b,-1)}{2} \right )} \\
% & = &\langle f_g, c_{\bar{r},\bar{a}} \rangle + C_g\,.
% \end{eqnarray*}
Note that $C_g$ depends on the statistical query $g$, but not on the target function.
We now wish to analyze the second term, $\langle f_g,c_{\bar{r},\bar{a}} \rangle$, more precisely. To this end, we define the following functions parameterized by $s \in \zo$:
\begin{eqnarray}
c_{\bar{r},\bar{a}}^s(x,i,b)= \left\{
\begin{array}{ll}
0 & \textrm{if $b \neq s$,}\\
c_{\bar{r},\bar{a}}(x,i,b)& \textrm{if $b=s$},
\end{array}\right.
\quad \mbox{and} \quad
f_g^s(x,i,b)= \left\{
\begin{array}{ll}
0 & \textrm{if $b \neq s$,}\\
f_g(x,i,b)& \textrm{if $b=s$}.
\end{array}\right.
\label{eq:def_c_f}
\end{eqnarray}

Recall that $\langle f_g,c_{\bar{r},\bar{a}} \rangle$  is a sum over tuples $(x,i,b)$. We can separate the sum into two pieces: one with tuples where $b=0$ and the other with tuples where $b=1$. Using the functions $c_{\bar{r},\bar{a}}^s, f_g^s$ just defined,  we
 can write $\langle f_g,c_{\bar{r},\bar{a}} \rangle = \langle f_g^0,c_{\bar{r},\bar{a}}^0 \rangle + \langle f_g^1,c_{\bar{r},\bar{a}}^1 \rangle$. Hence,
 \begin{equation}
\E[g(x,i,b,c_{\bar{r},\bar{a}}(x,i,b))] = C_g + \langle f_g^0,c_{\bar{r},\bar{a}}^0 \rangle + \langle f_g^1,c_{\bar{r},\bar{a}}^1 \rangle.\label{eq:3terms}
\end{equation}

The inner product $\langle f_g^1,c_{\bar{r},\bar{a}}^1 \rangle$ depends on the statistical query $g$ and on $\bar{r}$, but not on $\bar{a}$. Thus only the middle term on the righthand side of \eqref{eq:3terms} depends on $\bar a$.

Consider an SQ oracle $\mathcal{O}={\mathcal{O}}_{c_{\bar{r},\bar{a}},\DDD} $ that responds to every query $(g,\tau)$ as follows (recall that $\DDD$ is the uniform distribution):
$${\mathcal{O}}_{c_{\bar{r},\bar{a}},\DDD} (g,\tau) = \left \{ \begin{array}{ll}
C_g + \langle f_g^1,c_{\bar{r},\bar{a}}^1 \rangle & \mbox{ if } |\langle f_g^0,c_{\bar{r},\bar{a}}^0 \rangle | < \tau, \\
\E[g(x,i,b,c_{\bar{r},\bar{a}}(x,i,b))] & \mbox{ otherwise}.
\end{array}
\right. $$
If the condition $|\langle f_g^0,c_{\bar{r},\bar{a}}^0 \rangle | < \tau$ is met for all the queries $(g,\tau)$ made by the learner, then the SQ oracle $\mathcal{O}$ never replies with a quantity that depends on $\bar{a}$. We now show that this is typically  the case.

Extend the definition of $c^s_{\bar{r},\bar{a}}$ (Equation~\ref{eq:def_c_f}) to any $(r,a)\in \zo^{d}\times \zo$ by defining
$$c_{r,a}^0(x,i,b)= \left\{ \begin{array}{ll}
0 & \textrm{if $b=1$,}\\
c_{r,a}(x,i,b) \left(=(-1)^{\langle r,x\rangle+a}\right) & \textrm{if $b=0$.}
\end{array}\right. $$
Note that for $r,r'\in\zo^{d}$ and $a\in \zo$,
$$ \langle c^0_{r,a}, c^0_{r',a}\rangle = \left\{\begin{array}{ll}
1/2 & \mbox{if $r=r'$,} \\
0 & \mbox{if $r\not=r'$.} \end{array} \right. $$
We get that $\{c^0_{r,0}\}_{r\in\zo^d}$ is an orthogonal set of functions, and similarly with $\{c^0_{r,1}\}_{r\in\zo^d}$.
The $\ell_2$ norm of $c_{r,0}^0$ is $\|c_{r,0}^{0}\|= \sqrt{\langle  c_{r,0}^{0},c_{r,0}^{0} \rangle}=1/\sqrt{2}$, so the set
$\{\sqrt{2} \cdot c^0_{r,0}\}_{r\in\zo^d}$ is orthonormal. A similar argument holds for $\{\sqrt{2} \cdot c^0_{r,1}\}_{r\in\zo^d}$.

Expanding the function $f_g^0$ in the orthonormal set $\{\sqrt{2} \cdot c^0_{r,0}\}_{r\in\zo^d}$, we get:
$$\sum_{r \in \zo^d} \langle f_g^0,\sqrt{2}\cdot c^0_{r,0} \rangle ^2 \leq \|f_g^0\|^2 = \langle f_g^0,f_g^0 \rangle \leq 1/2 \,.$$
(The first inequality is loose in general because the set $\{\sqrt{2} \cdot c^0_{r,0}\}_{r\in\zo^d}$ spans a subset of dimension $2^d$ whereas $f_g^0$ is taken from a space of dimension $2^{d+\log d + 1}$).
Similarly,
$$\sum_{r \in \zo^d} \langle f_g^0,\sqrt{2}\cdot c^0_{r,1} \rangle ^2 \leq \|f_g^0\|^2 = \langle f_g^0,f_g^0 \rangle \leq 1/2.$$
Summing the two previous equations, we get
$$\sum_{(r,a) \in \zo^{d} \times \zo} 2 \cdot \langle f_g^0,c_{r,a}^0 \rangle^2 \leq 1\,. $$

Hence, at most $2^{2d/3-1}$ functions $c_{r,a}$ can have $| \langle  f_g^0,c_{r,a}^0 \rangle | \geq 1/2^{d/3}$.  Since $\bar r,\bar a$ was chosen uniformly at random we can restate this: for any particular query $g$, the probability that $c^0_{\bar r, \bar a}$ has inner product more than $1/2^{d/3}$ with $f_g^0$ is at most $2^{2d/3-1}/2^{d+1} = 2^{-d/3}$. This is true regardless of $a$: since $c_{r,0}^0=-c_{r,0}^0$, we have
 $| \langle  f_g^0,c_{r,0}^0 \rangle | = | \langle  f_g^0,c_{r,1}^0 \rangle |$, so the event that $| \langle  f_g^0,c_{\bar r,\bar a}^0 \rangle | \geq  1/2^{d/3}$ happens with probability at most  $2^{-d/3}$ over $\bar r$, for $\bar a=0,1$.

% was: For all $r \in \zo^{d}$, the correlation between $c^0_{r,0}$ and $c^0_{r,1}$ is $-1$, and for all $r,r' \in \zo^{d}$ s.t.\ $r\not=r'$ and all $a,a' \in \zo$, the correlation between $c^0_{r,a}$ and $c^0_{r',a'}$ is zero. We get that the set of functions $\{c^0_{r,0}\}$ (where $r\in\zo^d$) are orthogonal to each other, and similarly the set of functions $\{c^0_{r,0}\}$

% was:
%For any fixed $t \in \zo^d$, the norm of $c_{t_0}^0$ is defined as  $\|c_{t,0}^{0}\|= \sqrt{\langle  c_{t,0}^{0},c_{t,0}^{0} \rangle}=1/\sqrt{2}$.
%Therefore, all the functions in the set  $S=\{\sqrt{2}c^0_{r,0}\}$ are orthonormal. For the function $f_{g^0}$,
%$$\sum_{r \in \zo^d} \langle f_g^0,\sqrt{2}c^0_{r,0} \rangle ^2 \leq \|f_{g^0}\|^2 = \langle f_{g^0},f_{g^0} \rangle \leq 1/2.$$
%Similarly, the functions in the set $\{\sqrt{2}c^0_{r,1}\}$ are orthonormal. Therefore,
%$$\sum_{r \in \zo^d} \langle f_{g^0},\sqrt{2}c^0_{r,1} \rangle ^2 \leq \|f_{g^0}\|^2 = \langle f_{g^0},f_{g^0} \rangle \leq 1/2.$$
%This implies that
%$\sum_{(r,a) \in \zo^{d} \times \zo} 2 \langle f_{g^0},c_{r,a}^0 \rangle^2 \leq 1. $
%Hence, at most $2^{2d/3-1}$ functions $c_{r,a}$ can have $| \langle  f_{g^0},c_{r,a}^0 \rangle | \geq 1/2^{d/3}$.

Recall that the learner makes $t$ queries, $g_1\through g_t$. Let $Good$ be the event that $|\langle f^0_{g_i},c_{\bar{r},\bar{a}} \rangle | \leq 1/2^{d/3}$ for all $i \in [t]$ (i.e., the oracle can answer each of the queries independently of $\bar a$).
%Since $c_{\bar{r},\bar{a}}$ was chosen uniformly at random, we have that
Taking a union bound over queries, we have $\Pr[Good] \geq %1-t\cdot(2^{2d/3-1}/2^{d+1}) =
 1-t/2^{d/3+2}$ (where the probability is taken only over $\bar r$).

% The learner can nonadaptively learn $\bar{r}$ (as explained in Section~\ref{sec:admparity}).  So we can assume w.l.o.g.\ that the learner learns $\bar{r}$. We noted above that $\langle f_g^1,c_{\bar{r},\bar{a}}^1 \rangle$ is independent of $a$. Now, if all the queries are  made with tolerance at least  $1/2^{d/3}$, then conditioned on event $Good$ the oracle $\mathcal{O}$ returns an answer $C_{g_i} + \langle f_{g_i}^1,c_{\bar{r},0}^1 \rangle$ (= $C_{g_i} + \langle f_{g_i}^1,c_{\bar{r},1}^1 \rangle$) for every query $g_i$. Therefore, conditioned on $Good$ the oracle answers are independent of $a$. We will be crucially relying on this fact.

% Let $h$ be the output hypothesis of the learner.  {\em Weak} learning $\mparity$ nonadaptively over the uniform distribution is simple.
% After learning $\bar{r}$, the learner can easily guarantee that $h(x,i,1)=c_{\bar{r},\bar{a}}(x,i,1)$ holds for all inputs for which $b=1$ (that is, half of all inputs). This is because $c_{\bar{r},\bar{a}}(x,i,b)$ is independent of $\bar{a}$ when $b=1$. For the other half of the inputs ($b=0$), the weak learner can just set
% $h(x,i,0)=c_{\bar{r},0}(x,i,0)$ with probability $\frac 1 2$ and $h(x,i,0)=c_{\bar{r},1}(x,i,0)$ with
% remaining probability $\frac 1 2$.

We argued above that there is a valid SQ oracle which, conditioned on $Good$, can be simulated using $\bar r$ but without knowledge of $\bar a$, as long as all queries are made with tolerance $\tau \geq   1/2^{d/3}$ (as in the theorem statement).
To conclude the proof, we now argue that no nonadaptive {\em strong} learner exists for $\mparity$ over the uniform distribution. For that we concentrate on the $b=0$ half of the inputs, where the outcome of $c_{\bar{r},\bar{a}}(\cdot)$ depends on $a$. Let $h$ be the output hypothesis of the learner.
% Let $p_1$ be the fraction of inputs $(x,i,0)$ on which $c_{\bar{r},0}(x,i,0)$
% makes a mistake (i.e. disagrees with $c_{\bar{r},\bar a}(x,i,0)$), and let $p_2$ be fraction of inputs $(x,i,0)$ on which $c_{\bar{r},1}(x,i,0)$
% makes a mistake.
For any input $(x,i,0)$ we have $c_{\bar{r},0}(x,i,0)=-c_{\bar{r},1}(x,i,0)$. Thus either $c_{\bar{r},0}(x,i,0) \neq h(x,i,0)$ or $c_{\bar{r},1}(x,i,0) \neq h(x,i,0)$, and so some choice of $\bar a$ causes the error of $h$ to be at least $1/4$.
% In other words, for $p=p_1+p_2=\frac 1 2$ fraction of all $(x,i,b)$ inputs either $c_{\bar{r},0}^0(x,i,b) \neq h(x,i,b)$ or $c_{\bar{r},1}^0(x,i,b) \neq h(x,i,b)$.

Let $A$ be the event that $\trueerror(h,c_{\bar{r},\bar a}) \geq 1/4$. Because $Good$ depends only on $\bar r$, we can think of $\bar a$ as being selected after the learner's hypothesis $h$ whenever $Good$ occurs. Thus,
%
%From the previous discussions, it follows that
$\Pr[A \,| \,Good] \geq 1/2$. Using $\overline{Good}$ to denote the complement of the event $Good$, we get
\begin{eqnarray*}
\Pr[A] & = & \Pr[A \wedge Good] + \Pr[A \wedge \overline{Good}] \\
& \geq & \Pr[A \,|\, Good] \Pr[Good]  + 0 \geq \frac{1}{2}(1-t/2^{d/3+2}).
\end{eqnarray*}
Therefore,  $\Pr[\trueerror(h,c_{\bar{r},\bar{a}}) \geq 1/4] \geq  \frac{1}{2}(1-t/2^{d/3+2})$, as desired.
\end{proof}
\fi

%shiva-added L and R in BLR.
%sofya: moved this only to the full version: the paper was already written by the time we talked to them.
%shiva-we still use the BLR result. So I disagree.
%sofya: shiva, if you had a discussion with them that helped you with actual results in this paper (not explaining the implications of their paper), feel free to ack. them.
\ifnum\full=0
paragraph{Acknowledgments.} We thank Enav Weinreb for many discussions related to the local model, Avrim Blum and Rocco Servedio for discussions about related work in learning theory, and Katrina Ligett and Aaron Roth for discussions about~\cite{BLR08}.
\else
\section*{Acknowledgments}
We thank Enav Weinreb for many discussions related to the local model, Avrim Blum and Rocco Servedio for discussions about related work in learning theory, and Katrina Ligett and Aaron Roth for discussions about~\cite{BLR08}. We also thank an anonymous reviewer for useful comments on the paper and, in particular, for the simple proof of \thmref{vcoccam}.

\bibliographystyle{acm}
\bibliography{pacparity}

\appendix
\section{Concentration Bounds} \label{sec:appchern}
We need several standard tail bounds in this paper.
%two versions of the Chernoff bound in this paper: the standard multiplicative form an The first one is the standard additive Chernoff bound and the second is Hoeffding's extension to Chernoff bounds.

% \begin{theorem}[Chernoff Bounds~\cite{chern}] \label{thm:chern}
% Let $X_1,\dots,X_n$ be $n$ independent identical Bernoulli random variables with $\E[X_i]=\mu$. Then for every  $\delta >0$, $$\Pr \left [ \frac{\sum_i X_i}{n} \geq \mu + \delta \right ] \leq \exp(-D(\mu + \delta || \mu)n)$$  and $$Pr \left [ \frac{\sum_i X_i}{n} \leq \mu -\delta \right ] \leq \exp(-D(\mu - \delta || \mu)n),$$ where $$D(x||y) = x \log \frac x y + (1-x)\log \frac {1-x}{1-y}$$ is the Kullback-Leibler divergence between Bernoulli distributed random variables with parameters $x$ and $y$ respectively.
% \end{theorem}

\begin{theorem}[Multiplicative Chernoff Bounds (e.g.~\cite{chern,AV79})] \label{thm:chern}
Let $X_1,\dots,X_n$ be i.i.d.\,Bernoulli random variables with $\Pr[X_i=1]=\mu$. Then for every  $\phi \in (0,1]$,
$$\Pr \left [ \frac{\sum_i X_i}{n} \geq (1+\phi)\mu\right ] \leq \exp\left(-\frac{\phi^2 \mu n}{3}\right)$$  and
$$Pr \left [ \frac{\sum_i X_i}{n} \leq (1-\phi)\mu \right ] \leq \exp\left(-\frac{\phi^2\mu n}{2}\right)\,.$$ \end{theorem}

\begin{theorem}[Real-valued Additive Chernoff-Hoeffding Bound~\cite{hoeff}] \label{thm:hoeff}
Let $X_1,\dots,X_n$ be i.i.d.\,random variables with $\E[X_i]=\mu$ and $a \leq X_i \leq b$ for all $i$. Then for every $\delta > 0$, $$\Pr \left [ \left | \frac{\sum_i X_i}{n} -  \mu \right | \geq \delta  \right ] \leq 2\exp\left(\frac{-2\delta^2 n}{(b-a)^2}\right).$$
\end{theorem}

\begin{lemma}[Sums of Laplace Random Variables]\label{lem:laplace}
Let $X_{1},...,X_{n}$ be i.i.d.\,random variables drawn from  $\Lap(\lambda)$ (i.e., with probability density $h(x) = \frac{1}{2\lambda}\exp\left(-\frac{|x|}{\lambda}\right)$). Then for every $\delta>0$,
$$\Pr\left [  \left| \frac{\sum_{i=1}^n X_i}{n}  \right | \geq  \delta\right ] = \exp\left (-\frac{\delta^{2}n}{4\lambda^2}\right )\,.$$
\end{lemma}
The proof of this lemma is standard; we include it here since we were unable to find an appropriate reference.
\begin{proof} Let $S=\sum_{i=1}^n X_i$. By the Markov inequality, for all $t>0$,
 $$\Pr[S > \delta n] = \Pr[e^{tS} > e^{t\delta n}] \leq  \frac{\E[e^{tS}]}{e^{t\delta n}} =
 \frac{m_S(t)}{e^{t\delta n}},$$ where $m_S(t) = \E[e^{tS}]$ is the moment generating function of $S$.
 To compute $m_S(t)$, note that the moment generating function of $X\sim\Lap(\lambda)$ is $m_X(t) = \E[e^{tX}] = \frac{1}{1-(\lambda t)^{2}}$, defined for $0 < t < \frac 1 \lambda$. Hence $m_S(t) = (m_X(t))^n = (1-(\lambda t)^{2})^{-n} < \exp(n(\lambda t)^{2})$, where the last inequality holds for $(\lambda t)^{2} < \frac 1 2$.
We get that $\Pr[S > \delta n] \leq \exp(n((\lambda t)^{2} -
t\delta))$. To complete the proof, set $t = \frac \delta 2\lambda^2$ (note
that if $\delta<1$ and $\lambda > 1$ then $(\lambda t)^{2} =
(\frac \delta 2\lambda)^2 < \frac 1 2$). We get that $\Pr[S > \delta n] \leq
\exp\left(n\left(\left(\frac \delta 2\lambda\right)^{2} - \frac{\delta^2}2 \lambda\right)\right) =
\exp\left(-n\frac{\delta^2}4\lambda^2\right)$, as desired.
\end{proof}

\end{document}